\documentclass[a4paper]{article}

\usepackage[T1]{fontenc}
\usepackage[utf8]{inputenc}
\usepackage{lmodern}
\usepackage{textcomp}
\usepackage{microtype}
\usepackage{comment}
\usepackage[makeroom]{cancel}
\usepackage{mathrsfs}
\usepackage{caption}
\usepackage{subcaption}
\usepackage{pdfpages}

\usepackage[english]{babel}
\usepackage{csquotes}
\usepackage{tabularx}

\usepackage{graphicx}
\usepackage{setspace}
\usepackage[export]{adjustbox}

\usepackage{amsmath,amsthm, amssymb, amsfonts}
\usepackage{todonotes}
\usepackage{kantlipsum}
\usepackage{wrapfig}
\usepackage{paralist}
\usepackage{wasysym}
\usepackage{verbatim}
\usepackage{comment}
\usepackage{fullpage}
\usepackage{mathtools}
\usepackage{graphicx}
\usepackage{bbm}
\usepackage{wrapfig}
\usepackage{todonotes}
\usepackage{mathrsfs}
\usepackage{enumerate}
\usepackage{dsfont}
\usepackage{units}
\usepackage[]{algorithm2e}
\usepackage{ltablex}
\usepackage{varwidth}
\usepackage{pgfplots}
\usepackage{tikz}
\usepackage{subcaption}
\usetikzlibrary{positioning}
\usepackage{algorithmic}
\usepackage{accents}

\usepackage{xcolor}

\usepackage{mleftright}

\newcommand{\footnoteremember}[2]{\footnote{#2}\newcounter{#1}\setcounter{#1}{\value{footnote}}}
\newcommand{\footnoterecall}[1]{\footnotemark[\value{#1}]}

\usepackage{indentfirst}

\usepackage{sistyle} 
\newcommand\bSI[1]{{\small[\SI{}{#1}]}}
\usepackage{titlesec}
\makeatletter
\newlength\unitwdth
\newlength\numwdth
\newlength\tdima
\newcommand\SIdescr[2]{%
	\setlength\tdima{\linewidth}%
	\addtolength\tdima{\@totalleftmargin}%
	\addtolength\tdima{-\dimen\@curtab}%
	\addtolength\tdima{-\unitwdth}%
	\addtolength\tdima{-\numwdth}%
	\parbox[t]{\tdima}{%
		#1
		\leaders\hbox{$\m@th\mkern \@dotsep mu\hbox{\tiny.}\mkern \@dotsep mu$}%
		\hfill
		\ifhmode\strut\fi
		\makebox[0pt][l]{%
			\makebox[\unitwdth][l]{}%
			\makebox[\numwdth][r]{#2}}}}
\makeatother

\newcommand{\Z}{\mathbb{Z}}
\newcommand{\N}{\mathbb{N}}

\newcommand{\R}{\mathbb{R}}

\newcommand{\eps}{\epsilon}
\newcommand{\Realization}{\mathrm{R}}

\let\emptyset\varnothing

\newcommand{\RestrictedNN}[2]{{\mathrm{R}(\Phi_{#1})({#2})}}

\newcommand{\erisk}{\widehat{\mathcal{R}}}

\newtheorem{theorem}{Theorem}[section]
\newtheorem*{theorem*}{Theorem}
\newtheorem{remark}[theorem]{Remark}
\newtheorem{definition}[theorem]{Definition}
\newtheorem{proposition}[theorem]{Proposition}
\newtheorem{lemma}[theorem]{Lemma}

\newtheorem{assumption}[theorem]{Assumption}
\newtheorem*{remark*}{Remark}
\newtheorem*{proposition*}{Proposition}

\numberwithin{equation}{section}

\DeclareMathOperator{\suppp}{supp \,}

\newcommand{\pp}[1]{#1}

\newcommand{\vk}[1]{#1}
\newcommand{\ck}[1]{#1}

\definecolor{darkcandyapplered}{rgb}{0.64, 0.0, 0.0}

\newlength\mylen

\usepackage[pdftex,pdfpagelabels]{hyperref}
\hypersetup{pdftitle={},
	pdfauthor={}}

\title{Limitations of neural network training due to numerical instability of backpropagation%
}

\author{Clemens Karner~\footnoteremember{UniVieFakMath}{University of Vienna,
    Faculty of Mathematics.
	Kolingasse 14--16, 1090 Vienna, Austria.
    \texttt{\{clemens.karner,vladimir.kazeev,philipp.petersen\}@univie.ac.at}.
        VK was partially supported by the Austrian Science Fund (FWF) under project F65
        ``Taming Complexity in Partial Differential Systems''.
    }
	\and
	Vladimir Kazeev~\footnoterecall{UniVieFakMath}~\footnoteremember{UniVieDataSci}{University of Vienna,
    Research Network Data Science.
	Kolingasse 14--16,
    1090 Vienna, Austria.}
	\and
	Philipp Christian Petersen~\footnoterecall{UniVieFakMath}~\footnoterecall{UniVieDataSci}
}

\begin{document}

	\maketitle

	\begin{abstract}
		We study the training of deep neural networks by gradient descent where floating-point arithmetic is
		used to compute the gradients.
		In this framework and under realistic assumptions, we demonstrate that it is \emph{highly unlikely} to find ReLU neural networks that maintain, in the course of training with gradient descent, \emph{superlinearly} many affine pieces with respect to their number of layers.
		In virtually all approximation theoretical arguments which yield high order polynomial rates of approximation, sequences of ReLU neural networks with \emph{exponentially} many affine pieces compared to their
		numbers of layers are used.
		As a consequence, we conclude
		that approximating sequences of ReLU neural networks resulting from gradient descent in practice differ substantially from theoretically constructed sequences.
		The assumptions and the theoretical results are compared to a numerical study, which yields concurring results.


	\end{abstract}
	
	\par\medskip\noindent
	\textbf{Keywords:} : Round-off errors, deep neural networks, gradient descent, numerical stability.
	\par\smallskip\noindent
	\textbf{Mathematics Subject Classification:} 35A18, 65T60, 68T10.

	\section{Introduction}
Deep learning is a machine learning technique based on artificial neural networks which are trained by gradient-based methods and which have a large number of layers. This technique has
been tremendously successful in a wide range of applications~\cite{lecun2015deep, krizhevsky2012imagenet,silver2016mastering,senior2020improved}. 
Of particular interest for applied mathematicians are recent developments in which deep neural networks are applied to tasks of numerical analysis
such as the numerical solution of inverse problems~\cite{arridge2019solving,ongie2020deep, li2020nett, jin2017deep, pineda2022deep}
or of (parametric) partial differential equations~\cite{weinan2018deep,han2018solving,raissi2019physics,grohs2018proof,schwab2019deep,kutyniok2019theoretical,lu2021learning, bhattacharya2020model}.

The appeal of deep neural networks for these applications is due to their exceptional efficiency in representing functions
from several approximation classes that underlie well-established numerical methods.
In terms of approximation accuracy with respect to the number of approximation parameters,
deep neural networks have been theoretically proven to achieve approximation rates that are at least as good as those of finite elements~\cite{LinearFEMReLU, OPS19_811, marcati2020exponential},
local Taylor polynomials or splines~\cite{Yarotsky2017, ReLUSobolev},
wavelets~\cite{shaham2018provable} and, more generally, affine systems~\cite{boelcskeiNeural}.

In the sequel, we consider neural networks with the rectified-linear-unit (ReLU) activation function, which is standard in most applications.
In this case, the neural-network approximations are piecewise-affine functions. 
We point out that all state-of-the-art results on the rates of approximation with deep ReLU neural networks
that achieve higher order polynomial approximation rates
are based on explicit constructions 
with the number of affine pieces growing exponentially with respect to the number of layers;
see, e.g.,~\cite{Yarotsky2017,Telgarsky2015RepresentationBO}.

In this work, we argue that this central building block, functions with exponentially many affine pieces, \emph{cannot be learned with the state-of-the-art techniques}. 
We show theoretically that training in floating-point arithmetic is hindered by a bound on the number of affine pieces created during an iterative learning process. 
This bound is polynomial with respect to the number of approximation parameters and
also with respect to
the number of iterations.
Notably,
this non-computability of functions with exponentially many affine pieces is not derived from an abstract result of computability theory,
as the related results established in~\cite{boche2022limitations,colbrook2022difficulty};
instead, we identify a concrete reason why gradient descent based on backpropagation cannot produce these functions 
in floating-point arithmetic. The effect of numerical instability is also demonstrated in numerical examples. We stress that our results do \emph{not} imply that theoretically derived approximation rates cannot be realized with learned functions. Instead, we can merely conclude that the approximating sequences of neural networks found in practice must be fundamentally different than the theoretically derived ones.
We will give a detailed description of our findings in Subsection \ref{sec:contribution}.

This work is strongly influenced by~\cite{Hanin2019}, which shows that
functions with exponentially many affine pieces with respect to the number of layers typically do not appear in randomly initialised neural networks. 
The ideas presented in that work form the basis of our analysis. 

Before presenting our main results, we recapitulate the notion of floating-point arithmetic and furnish an example illustrating how
the finite precision of floating-point arithmetic can undermine learning due to the phenomenon known in numerical analysis as
\emph{catastrophic cancellation}.

\subsection{Floating-point arithmetic}\label{sec:floatingpointarithmetic}

Computations are performed almost exclusively in binary floating-point arithmetic, which
consists in restricting the arithmetic of real numbers to a discrete set of the form  
\begin{align}\label{eq:defOfM}
\mathbb{M}
= \big\{\pm 2^e \sum_{k=0}^{p} 2^{-k} c_k \colon c_0 = 1, \, c_1,\ldots,c_s \in \{0,1\}, \, e \in \{e_{\mathrm{min}}, \ldots, e_{\mathrm{max}}\} \big\},
\end{align}
extended with a few special elements, such as zero, infinity and NaN elements.
Here, the radix of the floating-point arithmetic is two, the
parameter $p\in\N$ is the precision and $e_{\mathrm{min}},e_{\mathrm{max}}\in\Z$ are the minimum and maximum exponents.
Naturally, $\mathbb{M}$ is not closed under basic arithmetic operations such as addition and multiplication, which necessitates rounding.
The round-to-nearest addition and multiplication are defined so that
\begin{equation}\label{eq:floating-point-ops}
	|a +_{\mathbb{M}} b - (a+b)| = \min_{s \in \mathbb{M}} |s-(a+b)|
	\, ,
	\qquad
	|a \times_{\mathbb{M}} b - a \cdot b| = \min_{s \in \mathbb{M}} |s-ab|
\end{equation}
for all $a,b \in \mathbb{M}$.
The precision $p$ and the above rounding together
define the so-called \emph{machine epsilon}
$\epsilon = \frac12 2^{1-p}$,
which is an accurate upper bound for the relative accuracy of the floating-point addition and multiplication:
\begin{equation}
	\label{eq:machine-eps-accuracy}
	| a +_{\mathbb{M}} b - (a + b)| \leq \epsilon \, | a + b |
	\quad\text{and}\quad
	| a \times_{\mathbb{M}} b - a \cdot b | \leq \epsilon \, | a \cdot b |
\end{equation}
for all $a,b \in \mathbb{M}$.
For the double- and single-precision IEEE~754 standards,
we have $\epsilon = 2^{-53} \approx 1.11 \cdot 10^{-16}$ and $\epsilon = 2^{-24} \approx 5.96 \cdot 10^{-8}$ respectively.
Note that computation in floating point arithmetic amounts to introducing a \emph{relative error} in each computation.

\subsection{Instability of deep neural networks due to catastrophic cancellation}\label{sec:firstexample}

The bounds~\eqref{eq:machine-eps-accuracy} allow for floating-point operations to be modelled
as perturbations of the respective infinite-precision operations of the real arithmetic.

Consider a feed-forward ReLU neural network with $L\in\N$ layers
with $d = N_0 \in \N$ real inputs
and $N_j \in \N$ neurons in each layer $j\in\N$. The evaluation
$\phi \colon \R^d \to \R^{N_L}$
of such a neural network
at a point $x^{(0)} \in \R^{d}$
consists of iteratively applying the transformations
\begin{align}
	\label{eq:simplifiedneuralNetwork}
	x^{(j)} = \varrho \, ( A_j \, x^{(j-1)} + b_j )
\end{align}
sequentially for $j  = 1,\ldots,L$,
which produces the corresponding output $\phi(x^{(0)}) = x^{(L)} \in \R^{N_L}$.
Here, $\varrho \colon \R \to \R$ given by $\varrho(x) = \max \{x,0\}$ for all $x \in \R$
is the ReLU activation function,
which is applied to the elements of $\R^{N_j}$ with each $j\in\{1,\ldots,L\}$
componentwise, whereas
$A_j \in \R^{N_{j-1} \times N_j}$ and $b_j \in \R^{\N_j}$ with $j\in\{1,\ldots,L\}$
are the weight matrices and bias vectors.

For simplicity, let us consider the case of $b_j = 0$ for all $j\in\{1,\ldots,L\}$
and assume that $x = x^{(0)} \in \R^d$ and $A_1,\ldots,A_L$ are such that all entries of $x^{(j)}$ are nonnegative
for each $j \in\{1,\ldots,L\}$.
Then
\begin{equation}
	\label{eq:nn-mat-prod}
	\phi(x) = A_L \cdots A_1 x,
\end{equation}
i.e., the evaluation of the network reduced to the multiplication of the input
by a matrix product consisting of $L$ factors.

The instability of such products for large $L$ was demonstrated in the context of
tensor decompositions,
in~\cite[Example~3]{BK:2020:StabPrec}:
the first component of the tensor considered therein
yields the matrix product~\eqref{eq:nn-mat-prod}
with $N_0=N_L = 1$, $N_1=\cdots=N_{L-1}=2$
and with weight matrices
\begin{align}\label{eq:example}
	A_1
	=
	\begin{pmatrix}
		(1+\varepsilon_1) \, a \,{} \\
		a
	\end{pmatrix}
	, \qquad
	A_L =
	\begin{pmatrix}
		(1+\varepsilon_L) (1+a^{-L}) \, a&-a \,{}
	\end{pmatrix}
	\quad\text{and}\quad
	A_{j}
	=
	\begin{pmatrix}
		(1+\varepsilon_j) \,a &0 \,{}\\
		0&a \,{}
	\end{pmatrix}
\end{align}
for each $j = 2, \dots, L-1$,
where $a > 1$ is a fixed parameter and $\varepsilon_1,\ldots,\varepsilon_L \geq 0$
are perturbation parameters.
For all $j \in \{1,\ldots,L\}$,
the case of
$\varepsilon_1 =\cdots=\varepsilon_{j-1}=\varepsilon_{j+1}=\cdots=\varepsilon_L=0$
and $\varepsilon_j \geq 0$
gives
$x^{(L)}(\varepsilon) = \phi(x^{(0)}) = 1+\varepsilon_j \, (a^L+1) \, x^{(0)}$ for any $x^{(0)}\in \R^d$.
Considering $\varepsilon$ as a perturbation parameter,
which may be at the level of
the machine epsilon $\epsilon$,
we observe that a relative perturbation of magnitude $\varepsilon$
in the $j$th weight matrix leads to the error
\[
x^{(L)}(\varepsilon) - x^{(L)}(0)
=
\varepsilon \, (a^L+1) \, x^{(L)}(0)
\, ,
\]
which implies the total loss of accuracy for any $\varepsilon>0$
whenever $a^L \gtrsim \varepsilon^{-1}$.
This effect is an immediate consequence
of the subtraction encoded by $A_L$, in which the two terms may be perturbed individually
(just as the first is perturbed by a multiplicative factor of $1+\varepsilon$ and
the second is not perturbed at all in the above example).
This is an archetypal example of the so-called \emph{catastrophic cancellation};
see, e.g.,~\cite[Section~1.7]{Higham:AccuracyAndStability-2nd}.

This example, based on inaccurate matrix multiplication,
can be generalised into the following statement for deep neural networks.
The precise statement is given as Proposition~\ref{prop:unstablenetComplete} in the Appendix.

\begin{proposition}\label{prop:unstablenetsimple}
	Let $L\in \N$ and $N \in \N$.  Let $a = 2^r \geq  1$ for $r \in \Z$ and let $\mathbb{M}$ be as in \eqref{eq:defOfM},
	{with $e_{\min}\leq 0$, $e_{\max} > 53+r$, $p= 53$} and hence $\eps = 2^{-53}$.
	If $(L-3)\log_{2}(N-1) + (L-1)\log_{2}(a-2\eps) \geq 54$, then there is a neural network $\Phi_{a, N,L}$ with $L$ layers, $N$ neurons per layer, and all weights bounded in absolute value below by 1 and above by $a$ such that  
	$$
	|\Phi_{a, N,L}(x) - \overline{\Phi_{a, N,L}(x)}| = |\Phi_{a, N,L}(x)|= |x|
	$$
	for all $x \geq 0$, where $\overline{\Phi_{a, N,L}(x)}$ is the neural network $\Phi_{a, N,L}$ evaluated in floating-point arithmetic with the machine epsilon $\eps$.
\end{proposition}

In the finite-precision setting, we observe that the evaluation of very small neural networks can already lead to a relative error of $1$ compared to the neural network with arbitrary precision. 
Admissible parameters such that the conclusion of Proposition \ref{prop:unstablenetsimple} holds are, for example, $N = 65$, $a = 10$, and $L = 8$.

Proposition \ref{prop:unstablenetsimple} considers only the forward application of a neural network.
However, the construction of this neural network is based on the example \eqref{eq:example} and is such that the neural network is entirely linear.
By the backpropagation rule, recalled in Definition \ref{def:GD},
the reversal of the neural network constructed in Proposition \ref{prop:unstablenetsimple} yields a neural network
such that
\emph{the accurate computation of the gradient of the weights with respect to the output of the neural network is infeasible in floating-point arithmetic}.

\subsection{Contribution}\label{sec:contribution}
In this work, we analyse the effect of inaccurate computations stemming from floating-point arithmetic in a more moderate regime than that of Proposition \ref{prop:unstablenetsimple}. In contrast to the error amplification, that was the basis of Example \ref{eq:example}, we assume the errors in the following to stem mostly from \emph{accumulation} of many inaccurate operations.  
	We assume that, in the computation of the gradient of a neural network, in each layer, a small and controlled relative error will appear. 
	Under this assumption, we observe
that neural networks trained with gradient descent will, with high probability, \emph{not} exhibit exponentially many affine pieces with respect to the number of layers.

We describe a framework for learning in floating-point arithmetic in Section 2. There, we define a gradient descent procedure where the gradient is computed with noisy updates.

To facilitate our main results, we make two assumptions on the gradient descent procedure (Assumptions \ref{assum:a} and \ref{assum:b}), which can be intuitively formulated as follows.

\begin{itemize}
	\item[(A)] We assume that the average of the reciprocals of the non-zero bias updates in an iteration is bounded by a polynomial in the maximum number $N$ of neurons in a single layer. 
	Moreover, every neuron with zero output (\emph{dead neuron}) is assumed to remain so after the iteration.
	Both assumptions together essentially require the gradient of every bias to be either zero or not too small for most neurons.
	We verify this assumption numerically in Section \ref{sec:numExp}.
	
	\item[(B)] For all layers $j \in \N$, the derivative of each coordinate of the preactivations, i.e., $A_j \, x^{(j-1)} + b_j$, with respect to the input of the neural network is bounded by a uniform constant.
\end{itemize}

Under these assumptions, we prove Theorem~\ref{thm:upperBoundOnPieces}, which can be intuitively phrased as follows: the expected number of affine pieces that the neural network has
after a single perturbed gradient descent update has an upper bound that is polynomial (in practice quadratic) in the number of neurons, linear in the number of layers and inversely proportional to the level of perturbation.

The result is based on the following insight of \cite{Hanin2019}. Consider as a prototypical neuron a function of the form $\varrho(h - b)$, where $h$ is a piecewise affine function of bounded variation. Then, for $\varrho(h - b)$ to have more affine pieces than $h$ it is necessary that $h(x) = b$ for some $x$ in the domain of $h$. Intuitively, for many such $x$ to exist, $h-b$ has to change sign many times. If $b$ is a random variable, then, for the event $h(x) = b$ to happen with high probability, $h$ needs to repeatedly cross a substantial part of the region where $b$ has most of its mass. This, however, can only happen to some degree since the variation of $h$ was assumed to be bounded. The resulting estimate of expected newly generated affine pieces in one neuron under random biases is formalised in Lemma \ref{lem:1}. 
Since the floating-point operations are assumed to introduce noise into the bias updates in a way quantified in Definition \ref{def:GD}, and since Assumption (B) implies bounded variation of the outputs of each neuron, we deduce that with high probability the number of affine pieces in a neural network is bounded. 

In a gradient descent iteration, the noise in each gradient update step also depends on the current empirical risk. Notably, if the neural network is initialised with zero-training error, no update happens. However, under standard assumptions on the dynamics of the loss, we can still prove that the number of affine pieces during the iteration does not increase quickly. This is the content of Theorem \ref{thm:numberofpiecesinfulliteration}, which we state below in a simplified and informal version:

\begin{theorem}[Informal version of  Theorem \ref{thm:numberofpiecesinfulliteration}]
	During gradient descent, where the gradient is computed by backpropagation in floating-point arithmetic, it holds for each iteration with probability $1/2$ that the number of affine pieces of a neural network along a given line
	grows at most polynomially (almost linearly) with respect to
	the number of iterations, polynomial in the number of neurons, and linear in the number of layers.
\end{theorem}

Assumptions (A) and (B) will be discussed in detail after their formal statement in Section \ref{sec:setup}. Since Assumption (A) relates to the distribution of updates, we numerically study if it is satisfied in practice in Section \ref{sec:numExp}.

Finally, in Section \ref{sec:numExpMain}, we numerically analyse the influence of the magnitude of round-off errors. We observe that, while lower numerical accuracy influences the generation of affine pieces negatively, as claimed by our main theorem, the effect is not as pronounced as
the direct application of the theorem may suggest.
The reason is that, for the approximations involved, the number of affine pieces is already significantly lower than our upper bounds.
Hence
we conclude that numerical stability is one reason that prevents the learning of
functions with exponentially many affine pieces,
but there may be many more such reasons that prevent the number of affine pieces from reaching the theoretically established threshold.
If the number of affine pieces is large for the initial neural network,
it reduces in the course of training;
in accordance with our theoretical results, that 
reduction occurs faster for lower computational accuracies.

\subsection{Related work}

This work represents a counter-point to many approximation theoretical approaches used in the analysis of deep learning. 
Here, we study to what extent piecewise-affine functions
with a number of affine pieces that is exponential with respect to the number of layers can be 
learned by deep neural networks using gradient descent algorithms.
Functions with high numbers of affine pieces are necessary in the approximation of the square function~\cite{Yarotsky2017}, which forms the basis for many approximation results in the literature, e.g.,~\cite{yarotsky2018optimal, OPS19_811, marcati2020exponential,PetV2018OptApproxReLU, schwab2019deep, ReLUSobolev}.

The primary motivation for the approach that we follow is taken from \cite{Hanin2019}. There it is shown that randomly initialised neural networks typically have few affine pieces. The main idea of that work is to show that to generate many affine pieces, the bias vectors in neural networks need to be very accurately chosen. This is unlikely to be achieved if the biases are initialised randomly. 
This very idea is also the basis behind our analysis. We cannot, however, simply apply the main results of~\cite{Hanin2019} since
analysing the whole iterative procedure of gradient descent requires keeping track of
the interdependence of the random variables modelling the weights and biases of the network
and estimating their variance throughout the process.

Floating-point arithmetic and its effect on iterated computations have been studied in the literature before.
The example given by~\eqref{eq:nn-mat-prod}--\eqref{eq:example}
is derived from the study of the stability of long tensor factorisations in~\cite{BK:2020:StabPrec}.
The explicit constructions of low-rank tensor approximations
presented in~\cite{KS:2017:QTTFE2d,KORS:2017:Multiscale1d,Marcati:2022:PointSing3D,KORS:2022:Multiscale}
all, depending on the specific implementation, may suffer from the instability similar to
that demonstrated in~\eqref{eq:nn-mat-prod}--\eqref{eq:example}.

In \cite{li2018exploration}, an empirical study is presented that finds that the accuracy affects the overall performance of a neural network classifier. In this context, also an effect of the number of layers is measured, showing that deeper neural networks are more sensitive to low accuracies. 
Another analysis is given in \cite{sun2022surprising}, where it is shown that numerical inaccuracies introduce stochasticity into the trajectories of trained neural networks, which is similar to noise in stochastic gradient descent. In the context of fixed-point arithmetic, \cite{gupta2015deep} showed that in many classification problems, one can restrict the computation accuracy significantly without harming the classification accuracy.
There is also a vast body of literature studying how neural networks can be efficiently quantised to speed up computations or minimise storage requirements \cite{zhou2017balanced, han2015deep, jacob2018quantization, hubara2017quantized}. Typically it is found that moderate quantisation or specifically designed quantisation does not harm classification accuracies.
To the best of our knowledge, the effect on the generation of many affine pieces has not been studied.

There have also been more abstract studies showing that on digital hardware or in finite precision computations, there are learning scenarios that cannot be solved to any reasonable accuracy; see~\cite{colbrook2022difficulty,boche2022limitations}.

	\section{Framework}\label{sec:setup}

	In this section, we introduce the framework for this article. We start by formally introducing neural networks in Definition \ref{def:NeuralNetworks}. 
Thereafter, we fix the notion of \emph{number of affine pieces} of a neural network in Definition \ref{def:numOfPieces}. 
To study the effect of numerical accuracy on the convergence and performance of gradient descent based learning, we define the perturbed gradient descent update in Definition \ref{def:GD} and the full perturbed gradient descent iteration in Definition \ref{def:finaldef}.

To facilitate the results in the sequel, we make Assumptions \ref{assum:a} and \ref{assum:b} on the gradient descent dynamics. 
Assumption \ref{assum:a} requires the absolute value of the derivative of the objective function with respect to each of the biases to either vanish or be bounded from below. 
In the event that the derivative vanishes, we assume that the associated neuron is dead, i.e., its output is constant on the whole domain. 
Assumption \ref{assum:b} stipulates that the output of each neuron should have a bounded derivative with respect to the input variable. 
We will discuss the sensibility of these assumptions at the end of this section.

We start by defining a neural network. Here we focus only on the most widely used activation function, the ReLU, $\varrho(x) \coloneqq \max\{0, x\}$, for $x \in \R$.
\begin{definition}[{\cite{PetV2018OptApproxReLU, OPS19_811}}]\label{def:NeuralNetworks}
	Let $d, L\in \N$. 
	A \emph{neural network (NN) with input dimension $d$ and $L$ layers} 
	is a sequence of matrix-vector tuples 
	\[
	\Phi \coloneqq \big((A_1,b_1),  (A_2,b_2),  \dots, (A_L, b_L)\big), 
	\]
	where $N_0 \coloneqq d$ and $N_1, \dots, N_{L} \in \N$, and 
	where $A_j \in \R^{N_j\times N_{j-1}}$ and $b_j \in \R^{N_j}$
	for $j =1,...,L$. The number $N_L$ is referred to as the \emph{output dimension}.
	
	For $j'\in \{1, \dots, L-1\}$, the \emph{subnetwork from layer $1$ to layer $j'$ of $\Phi$} is the sequence of matrix-vector tuples 
	\[
	\Phi_{j'} \coloneqq \big((A_{1},b_{1}),  (A_{2},b_{2}),  \dots, (A_{j'}, b_{j'})\big).
	\]
	
	For a NN $\Phi$ and a domain $\Omega \subset \R^d$, we define the associated
	\emph{realisation of the NN $\Phi$} as 
	\[
	\mathrm{R}(\Phi): \Omega \to \R^{N_L} : x\mapsto x^{(L)} := \mathrm{R}(\Phi)(x),
	\]
	where the output $x^{(L)} \in \R^{N_L}$ results from 
	\begin{equation}
		\begin{split}
			x^{(0)} &\coloneqq x, \\
			x^{(j)} &\coloneqq \varrho \, (A_{j} \, x^{(j-1)} + b_j) \quad \text{ for } j = 1, \dots, L-1,\\
			x^{(L)} &\coloneqq A_{L} \, x^{(L-1)} + b_{L}.
		\end{split}
		\label{eq:NetworkScheme}
	\end{equation}
	Here $\varrho$ is understood to act component-wise on vector-valued inputs, 
	i.e., for $y = (y^1, \dots, y^m) \in \R^m$,  $\varrho(y) := (\varrho(y^1), \dots, \varrho(y^m))$.
	We call $N(\Phi) \coloneqq d + \sum_{j = 1}^L N_j$ the \emph{number of neurons of the NN} 
	$\Phi$, $L(\Phi)\coloneqq L$ the \emph{number of layers} or \emph{depth}. 
	In addition, we refer to $(d, N_1, \dots, N_L)$ as the \emph{architecture of $\Phi$}. \ck{Furthermore, for a domain $\Omega$ and for $j \in \{1, \dots, L-1\}$, we define the \vk{\emph{preactivation functions}} $\eta_j \coloneqq \mathrm{R}(\Phi_j) \colon \Omega \to \R^{N_j}$.} 
\end{definition}

Since the ReLU is piecewise linear, it is not hard to see that the realisation of a NN is always a piecewise affine function. 
A widely-used tool to study deep NNs is by counting the number of the affine pieces. 
To formalise this, we present a definition of the number of affine pieces of a piecewise affine function.

\begin{definition} Let $d \in \N$. For a piecewise affine function $f:\Omega \to \R$ with $\Omega\subset\R^d$ and for any $\Omega'\subset\Omega$,
	we define the \emph{number of affine pieces of} $f$ on $\Omega'$ as the smallest integer $P$ such that there exist open sets $(U_i)_{i=1}^P \subset \R^d$ satisfying
	\begin{itemize}
		\item $\bigcup_{i=1}^P \overline{U_i} = \R^d$,
		\item $f_{|\Omega' \cap U_i}$ is affine for all $i \in \{1, \ldots , P\}$.
	\end{itemize}
	We denote the number of affine pieces of $f$ by $\mathrm{pieces}(f, \Omega')$. \label{def:numOfPieces}
\end{definition}

Next, we introduce a model for the training process of deep NNs. 
It is customary that NNs are trained by minimising an empirical risk function using gradient descent\footnote{In practice, stochastic gradient descent is often used, in which the gradient is only computed over random batches of training data instead of the complete dataset.
This introduces additional randomness into the gradient dynamics and,
in principle, should allow for our main results to be strengthened.
However, since this randomness does not stem from the numerical accuracy,
which is our main focus in the present work,
we confine it to the case of standard gradient descent.
}
over the weights. 
Assume we are given a loss function $\ell \colon \R^q \times \R^q \to \R^+$, for $q \in \N$, which could, for example, be the \emph{square loss} $\ell(y,y') = \|y-y'\|^2$.
Then,
for $M \in \N$, we define $\erisk \colon (\R^q)^M \times (\R^q)^M \to \R$ by
$$
	\erisk(Y, \hat{Y}) = \frac{1}{M}\sum_{i=1}^{M}\ell(Y_i,\hat{Y}_i)
$$ 
for all $Y = (Y_i)_{i=1}^M \subset \R^q, \hat{Y} = (\widehat{Y}_i)_{i=1}^M \subset \R^q$.

For NNs $\Phi$ with a fixed architecture, input dimension $d$, and output dimension $N_L= 1$, we formalise the procedure to minimise $\erisk(Y, \mathrm{R}(\Phi)(X))$ for given training data $X = ({x^{(L-1)}}i)_{i=1}^M \subset \Omega, Y = (Y_i)_{i=1}^M \subset \R$, where $\mathrm{R}(\Phi)(X) = (\mathrm{R}(\Phi)(X_i))_{i=1}^M$.

\begin{definition}[Gradient descent update of a neural network]\label{def:GD} 
	Let $d\in\N$, $\Omega \subset \R^d$, and let $\Phi$ be a NN of depth $L \in \N$, input dimension $d$, and output dimension $1$. Let $A_{j}, b_j, N_j$ be as in Definition \ref{def:NeuralNetworks}.
	Further, let $\varepsilon = (\varepsilon_j)_{j=1}^L >0$ be the \emph{sequence of effective relative perturbations} and $\lambda >0 $ be the \emph{step size.}
	Moreover, let $M \in \N$ and let $(x_i)_{i=1}^M \subset \Omega, (y_i)_{i=1}^M \subset \R$ be the training samples.
	
	Let $j \in \{1, \dots, L\}$. The \emph{exact gradient descent update of the biases in the $j$-th layer} is given by $u_j^b$, which is defined as
	\begin{align}
		u_j^b&\coloneqq\frac{1}{M}\sum_{i=1}^Mu_{j,i}^b, \text{ where} \label{eq:defuellb}\\
		u_{j,i}^b&\coloneqq I_j(x_i)  A_{j+1}^T I_{j+1}(x_i)A_{j+2}^T\cdots I_{L-1}(x_i) A_L^T \ck{\nabla_{\mathrm{R}(\Phi)(x_i)}\ell}(y_i,\mathrm{R}(\Phi)(x_i)) \, \text{ for } i \in \{1, \ldots, M\},\label{eq:defuellbis}
	\end{align}
	where $I_{j}(x) \in \{0, 1\}^{N_j}$ with $(I_{j}(x))_k = 1$ if and only if $\mathrm{R}(\Phi_j)(x) \geq 0$.  The \emph{exact gradient descent update of the weights in the $j$-th layer} is defined as
	\begin{align*}
		U_j^w&\coloneqq\frac{1}{M}\sum_{i=1}^M u_{j,i}^b \cdot \varrho(\RestrictedNN{j-1}{x_i})^T.
	\end{align*}
	Let $\Theta_j^b, \Theta_j^w \in \R^{N_j\times N_{j-1}}$ be two independent random variables of diagonal matrices containing i.i.d entries uniformly distributed on $[-0.5, 0.5]$. The \emph{perturbed bias and weight updates} are defined as
	\begin{align}
		\hat{u}_j^b&\coloneqq(I+\varepsilon_j\Theta_j^b)u_j^b \label{eq:update},\\
		\hat{U}^w_j&\coloneqq(I+\varepsilon_j\Theta_j^w)U_j^w. \nonumber
	\end{align}
	We define the \emph{updated exact weight and bias matrices} as
	\begin{align*}
		\bold{A}_{j}\coloneqq A_j - \lambda U_j^w, \quad
		\bold{b}_j\coloneqq b_j - \lambda u_j^b
	\end{align*}
	and the \emph{updated perturbed weight and bias matrices} as 
	\begin{align}
		\widehat{\bold{A}}_{j}\coloneqq A_j - \lambda\hat{U}_j^w,\quad
		\widehat{\bold{b}}_j\coloneqq b_j - \lambda\hat{u}_j^b. \label{eq:pertbiasweight}
	\end{align}
	The \emph{update of $\Phi$ with sequence of effective relative perturbations $\varepsilon$} is the random variable
	\begin{align*}
		\widehat{\Phi}^\varepsilon \coloneqq \big((\widehat{\bold{A}}_1,\widehat{\bold{b}}_1),  (\widehat{\bold{A}}_2,\widehat{\bold{b}}_2),  \dots, (\widehat{\bold{A}}_L, \widehat{\bold{b}}_L)\big).
	\end{align*}
	Moreover, for a domain $\Omega$ and for $j \in \{1, \dots, L-1\}$, we define the \emph{perturbed preactivations} as $\hat{\eta}_j \coloneqq \mathrm{R}(\widehat{\Phi}^\varepsilon_j) \colon \Omega \to \R^{N_j}$. 
\end{definition}
\begin{remark}\label{rem:EffectiveRelPerturbation}
	The perturbed updates of Definition~\ref{def:GD} are a model for numerical errors arising in floating-point arithmetic. In this model, the effective relative perturbations $\varepsilon_j$ introduced in Definition~\ref{def:GD} comprise all numerical errors resulting from repeated applications of matrix multiplications as well as from summing over potentially very large data sets.
The effective relative perturbations can therefore be assumed to be bounded from below by the machine precision. \pp{However, there are many reasons to assume that it may be \emph{significantly larger}.
\begin{itemize}
\item \textbf{Amplification.} As shown in Section~\ref{sec:firstexample}, unstable NNs may
	significantly amplify \vk{perturbations. The resulting accumulation of errors will affect lower layers more
	since the computation of the associated updates involves more matrix-vector multiplications.}
	As a result, we expect  $\varepsilon_j$ to guickly grow with $j \to 0$.
\item \textbf{Large data sets.} The computation of the update over large data sets requires the computation of
\vk{the mean of noisy, already} perturbed values. 
The computation of the mean \vk{leads to} further error amplification, which \vk{may be expected to grow with respect to dataset size}. This issue is also pointed out in the documention of the mean function of numpy, see \url{https://numpy.org/doc/stable/reference/generated/numpy.mean.html}.
\item \textbf{\vk{Forward pass}.} The computation of the updates in \eqref{eq:defuellbis}
is based on the current \vk{values $\ell(y_i,\mathrm{R}(\Phi)(x_i))$, $i \in \{1, \ldots, M\}$, of the loss function,}
as well as on the \vk{intermediate-layer outputs $\mathrm{R}(\widehat{\Phi}_j(x_i))$ with $i \in \{1, \ldots, M\}$ and $j \in \{1, \dots, L\}$}. In practice, both of these computations are also affected by numerical errors, which can be quite substantial \vk{(as seen in Section~\ref{sec:firstexample})}. This would add yet another amplification \vk{to $\varepsilon_1,\ldots,\varepsilon_L$}.
\item \textbf{Stochasticity.} The effective relative perturbation quantifies the level of stochasticity of the updates. In stochastic gradient descent, which is often used instead of \vk{the gradient descent described here}, updates \vk{are designed as random variables centered at} the true gradient. In this setting, perturbations as above arises naturally.
\end{itemize}

To illustrate the size of $\varepsilon$ compared to the machine precision, we present a numerical study in Section~\ref{sec:numExp}.
}
\end{remark}

\vk{A complete} iteration of gradient descent with perturbed updates corresponds to repeatedly applying Definition~\ref{def:GD}. We present a formal definition below.

\begin{definition} \label{def:finaldef}
	Let $d \in \N, \Omega \subset \R^d$, let $\Phi$ be a NN of depth $L \in \N$, input dimension $d$, and output dimension $1$. Further, let $M \in \N$, and let $X=(x_i)_{i=1}^M \subset \Omega, Y = (y_i)_{i=1}^M \subset \R$ be the training samples. 
	Let $\varepsilon$ be the sequence of effective relative perturbations and $(\lambda_n)_{n \in \N} \subset \R^+$ be the sequence of step sizes.
	
	We define a series of NNs by $(\widehat{\Phi}^{\varepsilon, n})_{n \in \N}$, where for $n \in \N$, $\widehat{\Phi}^\varepsilon_{n+1}$ is the update of $\widehat{\Phi}^{\varepsilon, n}$ with effective relative perturbation $\varepsilon$ and step size $\lambda_n$, and $\widehat{\Phi}^\varepsilon_1 = \Phi$.
	We call $(\widehat{\Phi}^{\varepsilon, n})_{n \in \N}$ the \emph{training sequence with initialisation} $\Phi$.
\end{definition}

Next, we present two crucial assumptions on the perturbed gradient descent updates of Definition \ref{def:GD}. 
First, we assume that the average of the reciprocals of the non-zero bias updates is not too large. In addition, we assume that if the update of the bias weight of a neuron is equal to zero, then the associated neuron is dead.
\begin{assumption}\label{assum:a}
	Let $d \in \N$ and $\Omega \subset \R^d$ be a domain. Let $\nu, c_0 >0$. Maintaining the notation of Definition \ref{def:GD}, we assume that $\sum_{j=1}^{L-1} \sum_{k=1}^{N_j} |(u_j^b)_k|^{\dagger} \leq N^\nu c_0$, where $x^\dagger = x^{-1}$ for $x > 0$ and $0^\dagger = 0$ and $N= d + \sum_{j=1}^L N_j$. 
	
	In addition, we assume for all $j\in \{1, \dots, L-1\}$ and $k \in \{1, \dots, N_j\}$ that $(u_j^b)_k=0$ implies $\hat{\eta}_{j}(x) \leq 0$ for all $x \in \Omega$.
	
\end{assumption}
The second assumption that we make is that the derivative of each of the preactivated neurons is bounded by one.
\begin{assumption}\label{assum:b}
	Let $d \in \N$ and $\Omega \subset \R^d$ be a domain. Maintaining the notation of Definition \ck{\ref{def:NeuralNetworks}}, we assume that for all $j \in \{1, \ldots, L-1\}$ and $k \in \{1, \ldots, N_j\}$, we have that $|(\nabla_x(\eta_j)_k)| \leq 1$ or all $x \in \Omega$. 
\end{assumption}

We end this section by discussing the sensibility of the Assumptions \ref{assum:a} and \ref{assum:b} above.
Assumption \ref{assum:a} requires the average of the reciprocals of the updates of the bias vectors to bounded by $c_0 N^{\nu-1}$.
We numerically check this assumption in Section \ref{sec:numExp} and find that for $\nu = 2$ the first part of the assumption is satisfied with probability at least $1/2$ for $c_0 < 0.1$. 
We also show theoretically that under mild assumptions on the distribution of $(u_j^b)_k$, we can expect Assumption \ref{assum:a} to hold for moderate $\nu$.
Moreover, 
when sufficiently many training samples are used,
any neuron such that the derivative with respect to its bias value is zero at all points in the training set is likely to be
dead. 
Since this implies that the associated bias will not be changed in the gradient descent step and since the bias value is the most important (but not only) parameter to determine whether a neuron lives, it is likely that the neuron remains dead after one gradient descent step.

On the other hand, Assumption~\ref{assum:b} requires the output of each neuron to have a gradient of length not exceeding one. 
It will be clear from the proofs of our main results that the bound of 1 can be replaced by any other positive number. 
Note that the boundedness of the gradient is reasonable to assume in general since a) this is a direct consequence of most initialisation schemes \cite{Hanin2019}, b) it is a sufficient condition to prevent the amplification effects that were the basis of the examples found in Subsection \ref{sec:firstexample} of the introduction, c) quite often an additional regularisation, such as weight decay \cite[Section 7.1]{Goodfellow-et-al-2016}, is used, which promotes smaller weights and implies bounds on the derivatives of the outputs of each neuron.

	\section{The number of affine pieces generated in perturbed gradient descent}

In this section, we demonstrate that, in the framework of Section \ref{sec:setup} and under Assumptions \ref{assum:a} and \ref{assum:b}, a neural network trained via gradient descent with an appropriately chosen step size will not admit a high number of affine pieces. 
We demonstrate in Theorem \ref{thm:upperBoundOnPieces} that the realisation of a NN after one step of gradient descent, in expectation, admits a number of affine pieces that scales polynomially with the number of neurons. Moreover, we show in Theorem \ref{thm:numberofpiecesinfulliteration}, that the number of affine pieces is polynomial in the number of gradient descent steps and in the number of neurons of the neural network. 
The polynomial dependence on the number of neurons in both results depends on the parameters in Assumption \ref{assum:a} and this polynomial is numerically identified to be quadratic in Section \ref{sec:numExp}.

Before we can state the results, we need some auxiliary results which will be collected in Subsection \ref{sec:oneneuron}. In Subsection \ref{sec:fullneuralNetwork}, we combine these results to obtain an upper bound on the expected number of affine pieces in a full neural network. 
Finally, we present a high-probability upper bound on the number of affine pieces during the full iteration of gradient descent in Subsection \ref{sec:fulliteration}. 

\subsection{Expected number of generated affine pieces in one neuron}\label{sec:oneneuron}

Intuitively speaking, a neuron of the form $\varrho(h + b)$,
where $h$ is a piecewise affine map mapping from $\R^d$ to $\R$
and where $b\in \R$, generates an affine piece
if $h(x) + b = 0$ for a point $x$ at which $h$ is affine and not constant. In all other cases, $h + b$ is either constant or the range of $h(\cdot) + b$ locally lies in one of the two linear regions of $\varrho$, which implies that the regularity of $\varrho(h + b)$ is the same as that of $h$. 

If $b$ in the argument above is chosen randomly, we expect that we can quantify the probability of generating a given number of affine pieces. 
The following lemma is a first step in this direction. To not disturb the flow of reading, the proof of this auxiliary result was deferred to Appendix~\ref{sec:auxResults}.

\begin{lemma}\label{lem:1}
	Let $c >0$ and $h$ be a piecewise-affine function on $[0,c]$ with $P \in \N$ affine pieces. Let $t \in \N$ and $A$ be a Lebesgue measurable set and assume that for every $y \in A$ it holds that 
	$$
	\#\{x \in [0,1] \colon h(x) = y \} \geq t.
	$$
	Then, $c \|h'\|_\infty \geq \|h'\|_1 \geq\lambda(A) \, t$, where $\lambda$ is the Lebesgue measure.
\end{lemma}

\begin{remark}\label{rem:ProbEstimateOnNumberOfPieces}
	Lemma \ref{lem:1} can be reformulated as saying that 
	$$
	\mathbb{P}(\#\{x \in [0,c] \colon h(x) = U \} \geq t) \leq \frac{c}{\delta t},
	$$
	where $U$ is a uniformly distributed random variable on an interval of length $\delta$ and $|h'| \leq 1$.  An interpretation of this is that the probability of having many intersections of a uniform random variable with a function of bounded gradient is bounded by the inverse width of the uniform distribution and the number of intersections.
\end{remark}

\begin{remark*}
	Note that the estimate in Lemma \ref{lem:1} is independent from the number of affine pieces $P$, but the proof requires that $P$ be finite. 
\end{remark*}

In the framework of Definition \ref{def:GD}, we have that the updates at each gradient descent step are randomly perturbed. It follows from Lemma \ref{lem:1}, that a random bias vector is unlikely to generate a high number of affine pieces. As a consequence, we obtain the following result, which is proved in Appendix \ref{sec:proofprop3}.

\begin{proposition}\label{prop:gd}
	Let $L,d\in \N$, $N_1, \dots N_{L-1} \in \N$,  $N_L = 1$, and let $\Phi$ be a NN with architecture $(d, N_1, \dots, N_L)$. Let $\widehat{\Phi}^\varepsilon$ be the NN after one backpropagation step with sequence of effective relative perturbation $\varepsilon=(\varepsilon_1,\ldots,\varepsilon_L)\in (0,\infty)^L$ satisfying Assumption \ref{assum:b} and step size $\lambda >0$. Let $\kappa$ be a line in $[0,1]^d$. For every $j \in \{1, \ldots, L-1\}$ and $k \in \{1, \ldots, N_j\}$, let $\omega_{j,k} \subset \kappa$ be a set that contains all breakpoints of $\mathrm{R}(\widehat{\Phi}^\varepsilon_j)_k$ restricted to $\kappa$. Furthermore, let $\hat{\omega}_{j,k} \subset \kappa$ be the smallest set that contains all breakpoints of $\varrho(\mathrm{R}(\widehat{\Phi}^\varepsilon_j)_k)$ restricted to $\kappa$.
	
	Then, we have for all $q \in \N$ that 
	\begin{align}\label{eq:probabilityEstimate}
		\mathbb{P}(\# \left(\hat{\omega}_{j,k} \setminus \omega_{j,k} \right) \geq q)
		\leq
		2\mathcal{L}(\kappa)
		\cdot
		(q \lambda \, \varepsilon_j \,  |(u_j^b)_k|)^{-1}
		\, ,
	\end{align}
	where $\mathcal{L}(\kappa)$ is the length of $\kappa$.
	Moreover, 
	\begin{align}\label{eq:expectationestimate}
		\mathbb{E}(\# \left(\hat{\omega}_{j,k} \setminus \omega_{j,k} \right) ) \leq  \frac{2\mathcal{L}(\kappa)}{\lambda \, \varepsilon_j }
		\,
		|(u_j^b)_k|^{-1}
		\,
		\ln(N^{j+1})
		\, .
	\end{align}
\end{proposition}

\subsection{Expected number of affine pieces of neural networks after one gradient descent step} \label{sec:fullneuralNetwork}

Next, we bound from above the expected number of affine pieces of a full NN after a
single gradient descent update according to Definition~\ref{def:GD}. 
In Proposition \ref{prop:gd}, we bounded from above the expected number of added affine pieces by one neuron. 
The expected number of affine pieces added in the full NN is found by adding all the added affine pieces corresponding to individual neurons. 
The following theorem then follows by the linearity of the expected value. 
We present a proof in Appendix~\ref{sec:proofOfupperBoundOnPieces}.

\begin{theorem}\label{thm:upperBoundOnPieces}
	Let $L,d\in \N$, and $N_1, \dots N_{L-1} \in \N$,  $N_L = 1$, and let $\Phi$ be a NN with architecture $(d, N_1, \dots, N_L)$. Further, let $\widehat{\Phi}^\varepsilon$ be a NN after one backpropagation step with sequence of effective relative perturbation $\varepsilon \in (0,\infty)^L$ satisfying Assumption \ref{assum:a} with $c_0, \nu >0$ and Assumption \ref{assum:b} and with
	step size $\lambda >0$. 
	Then, for every line $\kappa \subset [0,1]^d$ of length $\mathcal{L}(\kappa)$,
	we have for every $j' = 1, \dots, L$
	\begin{align} 
		\mathbb{E}(\mathrm{pieces}(\mathrm{R}(\widehat{\Phi}^\varepsilon), \kappa)) \leq \left(1+ \frac{2 c_0}{\lambda} \, \mathcal{L}(\kappa) \, j' \, \hat{\varepsilon}_{j'}^{-1} N^\nu \ln(N) \right)  \cdot (2 N)^{L-j'}, \label{eq:statementwithjprime}
	\end{align}
	where $\hat{\varepsilon}_{j'} \coloneqq \min \{ \epsilon_1,\ldots,\epsilon_{j'} \}$.
\end{theorem}
\begin{remark}
	To interpret \eqref{eq:statementwithjprime}, we note that depending on the choice of $j'$ the upper bound varies between exponential in $L$ (for $j' = 1$) and linear in $L$ (for $j'= L$). Moreover, if $L-j'$ is fixed, then the whole estimate is a polynomial upper bound on the number of affine pieces in terms of the number of neurons $N$. 
	
	On a theoretical level, this demonstrates that the asymptotic scaling of the number of affine pieces of a NN after one backpropagation step does, in expectation, not scale exponentially with the number of layers. 
	
	For practical purposes, it should be mentioned that this upper bound is void if $\epsilon_j$ is too small for some $j < j'$. However, as described in Remark \ref{rem:EffectiveRelPerturbation}, we expect that, for NNs with many neurons per layer, already for reasonably small $j'$, performing $j'$ matrix-vector multiplications as in \eqref{eq:defuellb} as well as summations over large data sets should result in the accumulation of errors such that 
	the effective perturbations $\epsilon_j$ with  $j < j'$ are sufficiently large.
\end{remark}
Theorem \ref{thm:upperBoundOnPieces} demonstrates that the number of affine pieces that are expected after one gradient descent step is polynomial in the number of layers. 
This is problematic for constructions of NNs that create an exponential number of affine pieces with respect to the number of layers:
\emph{when the effective perturbations are not too small}, it is unlikely that a NN with a number of affine pieces
exponentially large with respect to the number of layers
is found after a single gradient descent step.

\subsection{Generation of affine pieces during full iteration}\label{sec:fulliteration}

We now take a closer look at the implications of Theorem \ref{thm:upperBoundOnPieces} when optimising a NN using gradient descent. As a consequence of Theorem \ref{thm:upperBoundOnPieces} and Markov's inequality for a given line $\kappa$ it holds that
\begin{align}\label{eq:theGoldenEquation}
	\mathbb{P}\left( \mathrm{pieces}(\mathrm{R}(\widehat{\Phi}^\varepsilon), \kappa) \geq 2 \min_{1 \leq j' \leq L} \big(1+   \frac{2 c_0}{\lambda} \, \mathcal{L}(\kappa) \, j' \hat{\varepsilon}_{j'}^{-1} N^\nu \ln(N) \big)  \cdot (2 N)^{L-j'}\right) \leq \frac{1}{2}.
\end{align}
We continue by studying the impact of Equation \ref{eq:theGoldenEquation} when training a NN with the gradient descent algorithm and obtain the following main result.
\begin{theorem}\label{thm:numberofpiecesinfulliteration}
	Let $d \in \N$ and $\Omega \subset [0,1]^d$. 
	 Let $\Phi$ be a NN of depth $L \in \N$ and input dimension $d$ and output dimension $1$, $M \in \N$ and $X=(x_i)_{i=1}^M \subset \Omega, Y = (y_i)_{i=1}^M \subset \R$ be the training samples. 
	Let $\varepsilon >0$ be the effective relative perturbation and $(\lambda_n)_{n \in \N}$ be the sequence of step sizes. Let $(\widehat{\Phi}^{\varepsilon, n})_{n \in \N}$ be the associated \emph{training sequence with initialisation} $\Phi$. 
	Assume that $\kappa \subset [0,1]^d$ is a line.
	
	Then, for each $n\in \N$ satisfying Assumption~\ref{assum:a} with $c_0 = c_0^{(n)}$ and Assumption \ref{assum:b} in the update step of $\widehat{\Phi}^{\varepsilon, n}$, we have that 
	\begin{align*}
		\mathbb{P}\left( \mathrm{pieces}(\mathrm{R}(\widehat{\Phi}^{\varepsilon, n}), \kappa) \geq 2 \min_{1 \leq j' \leq L} \Big(1+  \frac{2 c_0^{(n)}}{\lambda_n} \, \mathcal{L}(\kappa) \, j'
		\,
		\hat{\varepsilon}_{j'}^{-1}
		\,
		N^\nu \ln N \Big)  \cdot (2 N)^{L-j'}\right) \leq \frac{1}{2}
		\, .
	\end{align*}
\end{theorem}

\begin{remark}\label{rem:scalingOfC0}
	As presented in \cite{nemirovskii1983,nemirovski2009}, gradient descent as used in most machine learning applications typically achieves a convergence order of $1/\sqrt{n}$. 	
	
	Consider a step-size rule where $\lambda_n \cong n^{-\alpha}$, where $\alpha \in [0,1]$ e.g., $\alpha = 1/2$ as in \cite[Section 14.4.2]{shalev2014understanding} or $\alpha = 1$ as in \cite{nemirovski2009}.
	
	Considering the square loss, we can assume for a converging training sequence $(\widehat{\Phi}^{\varepsilon, n})_{n\in \N}$ that 
	\begin{align}\label{eq:lowerboundOnRisk}
		\erisk(R(\widehat{\Phi}^{\varepsilon, n}(X), Y)) \gtrsim \frac{1}{{n}} \text{ and }	\erisk(R(\widehat{\Phi}^{\varepsilon, n}(X), Y))  \to 0 \text{ for } n \to \infty.
	\end{align}
	The scaling law of \eqref{eq:lowerboundOnRisk} shows that it is reasonable to assume that 
\begin{align}\label{eq:scalingOfLoss}
		\lambda_n \, \bigg|\frac{1}{M}\sum_{i=1}^M \ell'(y_i, \mathrm{R}(\widehat{\Phi}^\varepsilon_{n-1}))\bigg| \gtrsim n^{-3 /2}
		\, .
\end{align}
	Indeed, assuming the reverse inequality of \eqref{eq:scalingOfLoss} shows by a direct computation (plugging \eqref{eq:scalingOfLoss} into \eqref{eq:defuellbis}) that the weights of the NNs $(\widehat{\Phi}^{\varepsilon, n})_{n\in \N}$ would converge with a rate faster than $n^{-1/2}$. This, by the local Lipschitz property of the realisation function \cite[Proposition 4.1]{petersen2021topological} yields a pointwise convergence of $(\mathrm{R}(\widehat{\Phi}^{\varepsilon, n}))_{n\in \N}$ to a limit with a rate faster than $n^{-1/2}$. Due to the quadratic dependence of the square loss on the pointwise error, this would violate \eqref{eq:lowerboundOnRisk}. 
	
	As a consequence of \eqref{eq:scalingOfLoss} as well as the linearity of \eqref{eq:defuellb} and \eqref{eq:defuellbis}, it is sensible to claim that $c_0^{(n)}$ of Assumption \ref{assum:a} should scale as $(n^{-3/2} / \lambda_n)^{-1} = n^{\alpha + 3/2}$
	with respect to
	$n \in \N$. For other loss functions than the square loss, similar arguments can be made, that may yield different exponents depending on the modulus of continuity of the loss.
	
	Based on \eqref{eq:defuellbis}, we, therefore, conclude that an admissible scaling law for Assumption \ref{assum:a} for $n \in \N$ is 
	\begin{align*}
		\lambda_n^{-1}c_0^{(n)} \lesssim n^{3/2}.
	\end{align*}

	If we assume a different lower bound for the convergence of gradient descent than $1/\sqrt{n}$, which could be sensible in special cases, then the scaling of $c_0^{(n)}$ would change accordingly.
	\end{remark}
	
	We observe that, if $\lambda_n c_0^{(n)}$ scales polynomially with respect to $n \in \N$, then Theorem~\ref{thm:numberofpiecesinfulliteration} states that NNs
	obtained by training with
	gradient descent, with probability $1/2$, will not
	have a number of affine pieces exponentially large
	with respect to the number of layers, unless an exponential number of iterations is performed. Remark~\ref{rem:scalingOfC0}
	shows
	that $c_0^{(n)}$ can be assumed to grow not faster than $n^{3/2}$
	with respect
	to $n \in \N$.

	Since, by construction, the NNs $(\widehat{\Phi}^{\varepsilon, n})_{n \in \N}$ are independent random variables, it holds that \emph{the event that the sequence $(\mathrm{R}(\widehat{\Phi}^{\varepsilon, n}))_{n \in \N}$ maintains many affine pieces for a long series of gradient descent steps is highly unlikely}.

	\section{Experimental studies} \label{sec:numExp}

\subsection{Size of the effective relative perturbation}

In Definition \ref{def:GD}, we assume that the gradient descent updates are perturbed by an effective relative perturbation $\varepsilon$. 
Theorem \ref{thm:upperBoundOnPieces} then shows that the expected number of pieces is bounded by the reciprocal of $\varepsilon$. Therefore, we would like to assume that $\varepsilon$ is not extremely small. In 
Remark \ref{rem:EffectiveRelPerturbation}, we argue that it is sensible to assume that $\varepsilon$ is substantially larger than the machine precision, especially in those coordinates that correspond to lower layers. In this subsection, we present numerical results that support this claim in practice. 

\begin{figure}[htb]
	\centering
	\includegraphics[width = 0.49\textwidth]{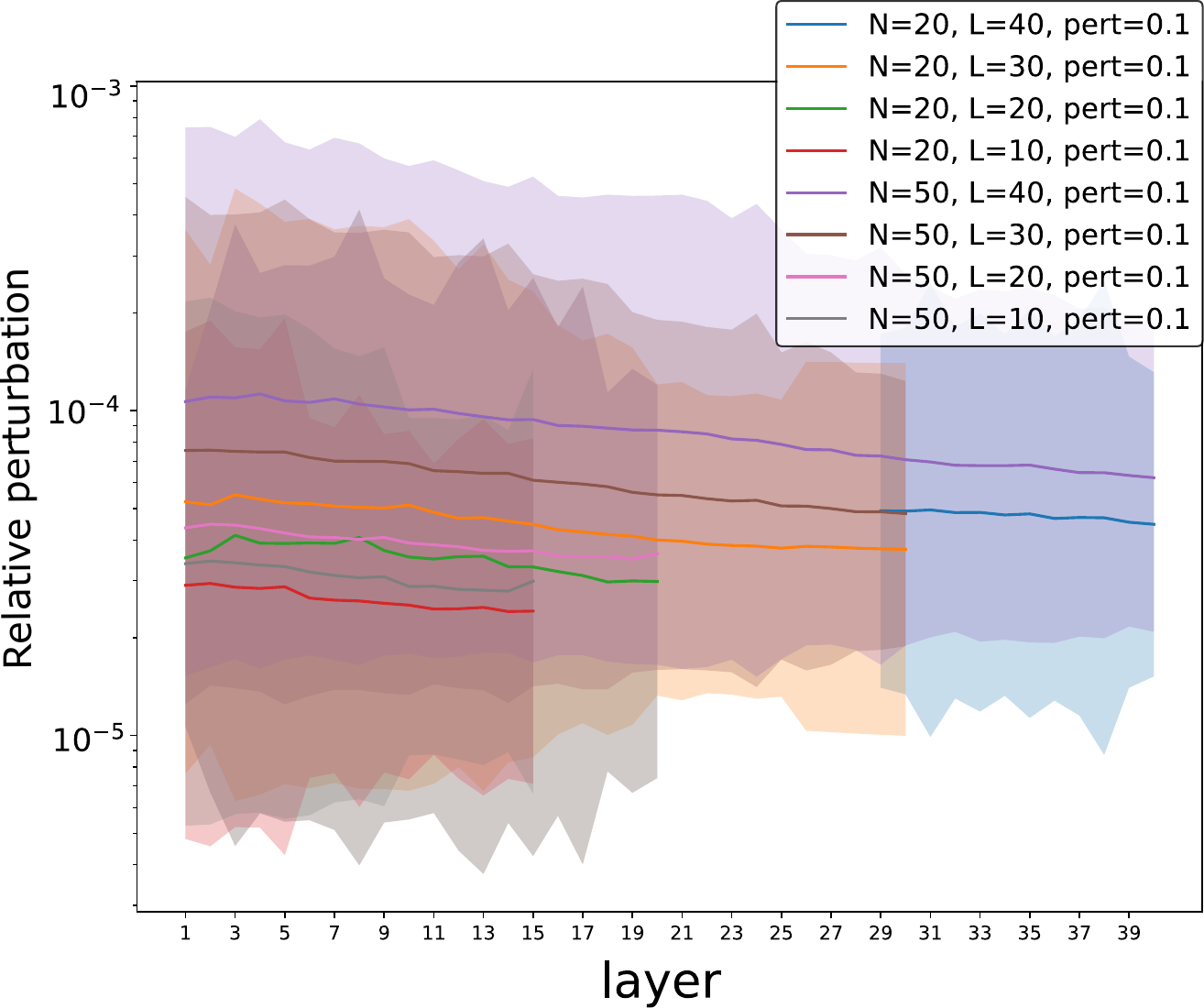}
	\includegraphics[width = 0.49\textwidth]{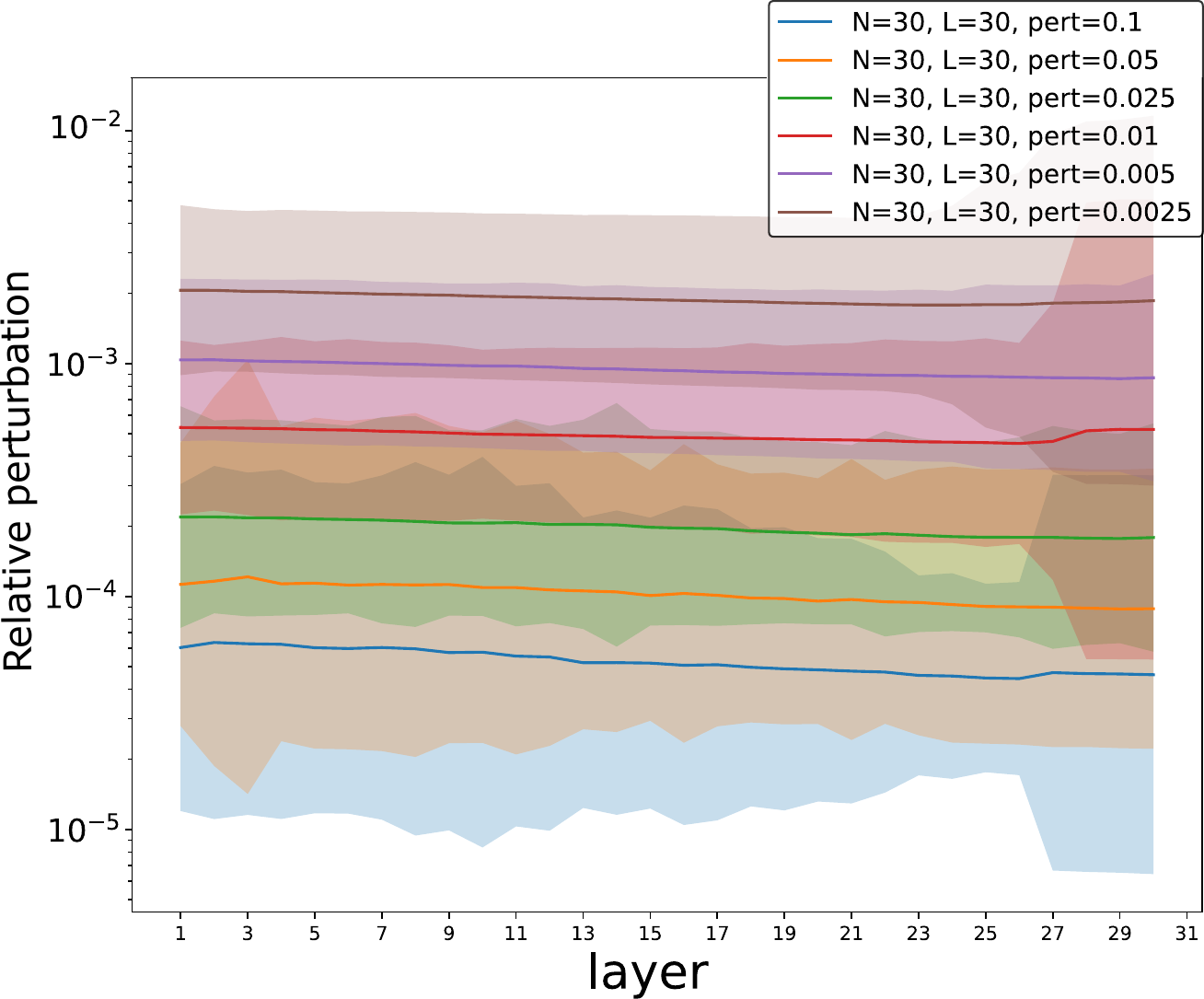}
	\caption{Mean size of $\varepsilon = (\varepsilon_j)_{j=1}^L$ for the experiment described above for a variety of architectures and perturbations. All results are averaged over 10000 runs, relative change of one standard deviation is shown as shaded area.}
	\label{fig:expVareps}
\end{figure}

We carry out the following experiment. We initialize a random neural network according to the He initialization \cite{he2015delving} and carry out one gradient descent step over a random training set of 100 data points. The training data are chosen randomly but such that they result in a $\ell_2$-norm perturbation of the data points of the size of $pert$.

We compute the forward and backward propagation as well as the computation of the involved means using computations with single precision and compare the relative error when computing the gradients associated to the bias vectors with an update using double precision. 
Then, for each layer, we compute the median value of the relative errors. This instead of taking the mean, is done so that the error will not be dominated by a few very large updates. 

The results are averaged over 10000 runs and shown in Figure  \ref{fig:expVareps}. In addition to the mean value over 10000 runs, a shaded region corresponding to one relative standard deviation is shown.
 
We observe that, the computation in single precision loses up to five orders of magnitude against its machine precision, which is of the order of $1e-8$. Architectures with more layers lead to higher errors. Also the relative error, seems to scale inversely with the norm of the loss.

Overall, the claim that $\varepsilon$ can often be considerably larger than the underlying machine precision appears warranted.

\subsection{Assumptions on gradient descent}\label{sec:GradientDescentAssumptions}

	\begin{figure}[htb]
	\centering
	\includegraphics[width = 0.7\textwidth]{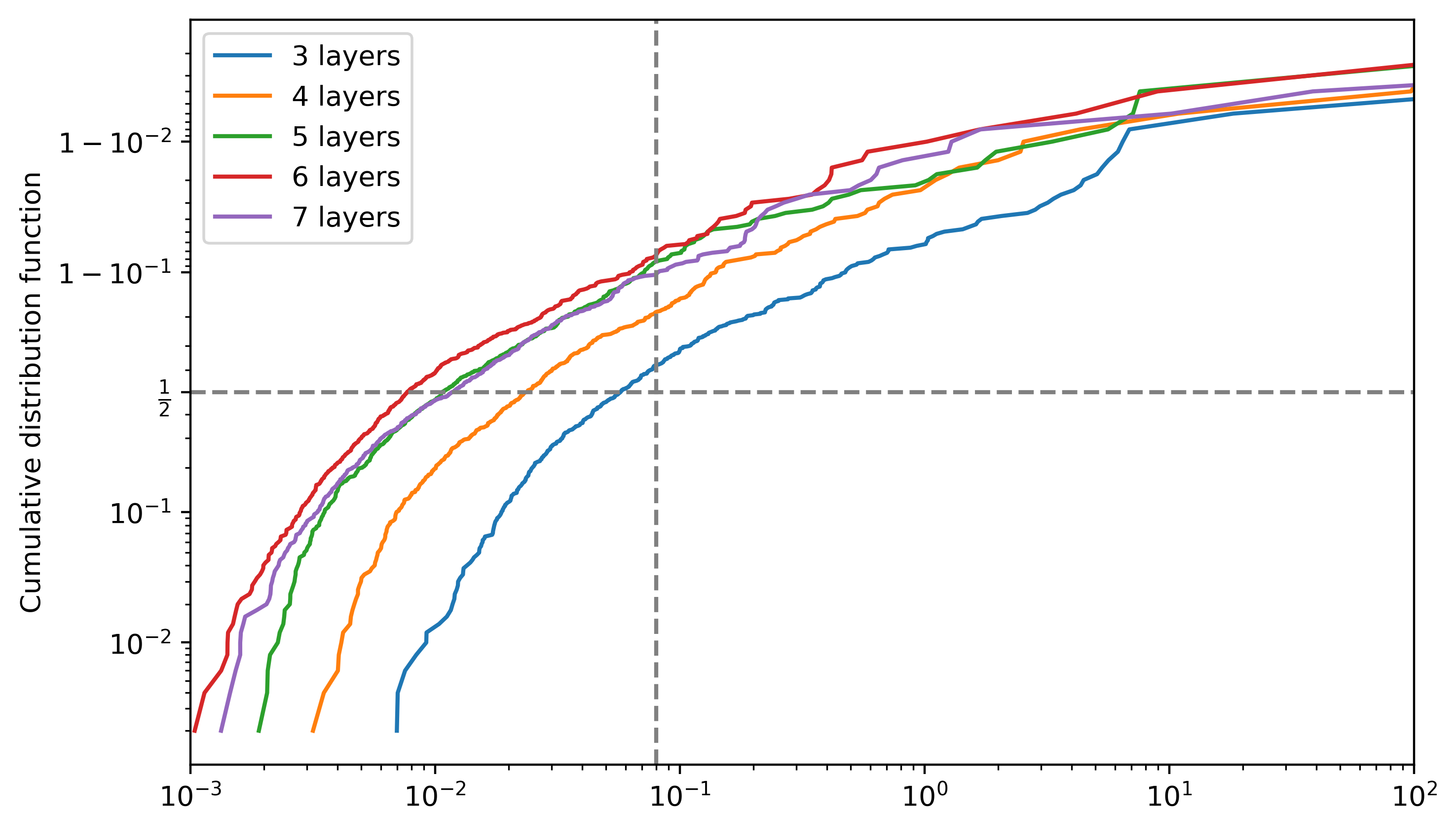}
	\put(-200, -10){$N^{-2}\,\sum_{j=1}^{L-1} \sum_{k=1}^{N_\ell} |(u_j^b)_k|^{\dagger}$}
	\caption{Cumulative distribution function for results of the experiment to estimate the value of \eqref{eq:theseAreTheValuesWeCareAbout} in practice. In about half of the cases the value \eqref{eq:theseAreTheValuesWeCareAbout} is bounded by $0.1$.}
	\label{fig:experiment1}
\end{figure}

\begin{figure}[htb]
	\centering
	\includegraphics[width = 0.4\textwidth]{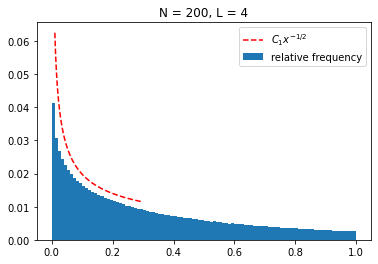} \  \includegraphics[width = 0.4\textwidth]{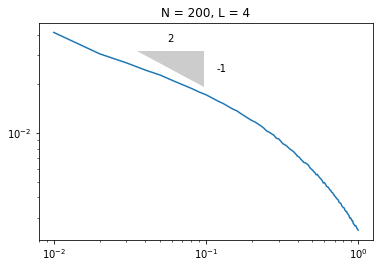}
	
	\includegraphics[width = 0.4\textwidth]{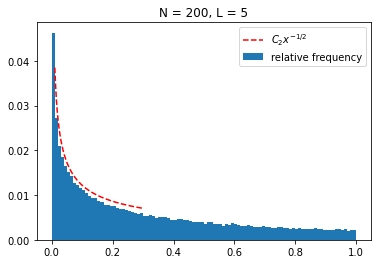} \ \includegraphics[width = 0.4\textwidth]{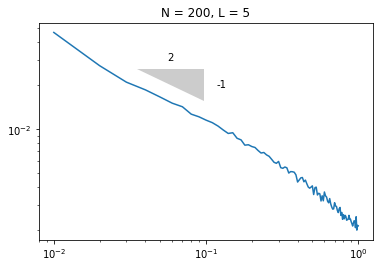}
	
	\includegraphics[width = 0.4\textwidth]{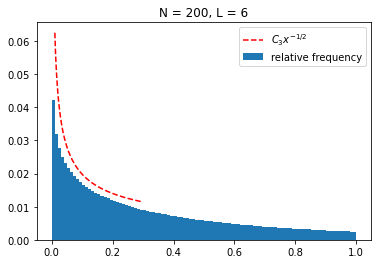} \ \includegraphics[width = 0.4\textwidth]{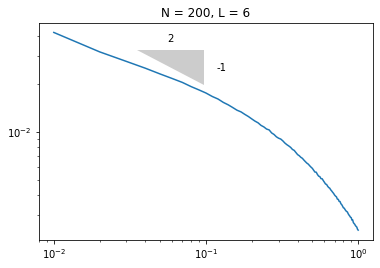}
	\caption{%
	Relative frequencies of the bias updates $(u_j^b)_k$ computed over 10'000 neural network instantiations
	with $N = 200$ neurons per layer and, from top to bottom, $L = 4,5,6$ layers. We only include updates with absolute value less than one since we are interested in the asymptotic behaviour at zero. The constants $C_1,C_2,C_3$ in the plots are $1/160, 1/270, 1/160$, respectively. The right panel shows the same distribution in a log-log plot and indicates that the density of the bias updates has a mild algebraic pole, of order $1/2$, at the origin
		since the slope of the distribution appears to be not steeper than $-1/2$.}
	\label{fig:behaviouratzero}
\end{figure}

	To verify the sensibility of Assumption \ref{assum:a}, we numerically check if the gradient descent update of the biases of a randomly initialised NN is on average appropriately far away from zero.
	For $L = 3, 4, 5, 6, 7$, we construct random NNs with $L$ layers and input and output dimensions equal to $1$ and with various numbers of neurons between $100$ and $400$ per layer. We initialise the NNs according to the He initialiser \cite{he2015delving}. Next, we choose $500$ random training points $(x_i)_{i=1}^{500}$ uniformly in $[0, 2\pi]$ and associate as labels $(\sin(x_i))_{i=1}^{500}$. We perform one step of gradient descent and collect the size of the updates of the bias vectors. 
	Finally, we compute for each NN the value 
	\begin{align}\label{eq:theseAreTheValuesWeCareAbout}
			\frac{1}{N^2}\sum_{j=1}^{L-1} \sum_{k=1}^{N_\ell} |(u_j^b)_k|^{\dagger},
	\end{align}
	where $N\in \N$ is the total number of neurons of the NN and $(u_j^b)_k$ are the updates of the bias vectors as in Assumption \ref{assum:a}. The term \eqref{eq:theseAreTheValuesWeCareAbout} corresponds to $\nu = 2$ in Assumption \ref{assum:a}. 
	
	In Figure \ref{fig:experiment1}, we depict for each number of layers, the resulting values of \eqref{eq:theseAreTheValuesWeCareAbout}. More precisely, for number of layers $L$ and for each number $N_1 \in \{100, 110, \dots, 390, 400\}$, we initialise $500$ neural networks with architecture $(1, N_1, \dots, N_1, 1)$ and compute the associated values of \eqref{eq:theseAreTheValuesWeCareAbout}. We average these values over the 500 runs. Then, we depict in Figure \ref{fig:experiment1}, the maximum value of these averages for all values $N_1$. We observe that, in at least half of the cases the values \eqref{eq:theseAreTheValuesWeCareAbout}, are bounded by $c_0 < 0.1$.

	Theoretically, we can also justify Assumption \ref{assum:a}, if we assume that all bias updates are independent identically distributed random variables with compactly supported densities $\sigma$ with $\sigma(x) \in \mathcal{O}(x^{-(1-\alpha)})$ for $x \to 0$, where $\alpha \in (0,1)$.
	
	In this case, we have that for a $q > 1/\alpha$ with $q \in \N$, 
	\begin{align*}
		\left(\frac{1}{N^{q}}\sum_{j=1}^{L-1} \sum_{k=1}^{N_\ell} |(u_j^b)_k|^{\dagger}\right)^{1/q}  \leq \frac{1}{N}\sum_{j=1}^{L-1} \sum_{k=1}^{N_\ell} \left(|(u_j^b)_k|^{\dagger}\right)^{1/q}
	\end{align*}
	by the binomial theorem. Moreover, for large $N \in \N$, we will have by the law of large numbers that with high probability
	\begin{align*}
		\frac{1}{N}\sum_{j=1}^{L-1} \sum_{k=1}^{N_\ell} \left(|(u_j^b)_k|^{\dagger}\right)^{1/q} \approx \int_{(0, \infty)} x^{-1/q} \sigma(x) dx \lesssim \int_{(0, c)} x^{-1/q + \alpha - 1} dx < \infty,
	\end{align*}
	where $c$ is such that $\suppp \sigma \subset [0, c]$ and we used that $-1/q + \alpha>0$.	
	
	We can study the distribution of $(u_j^b)_k$ numerically to see if it, indeed, has a mild pole at zero. We use the same setup as in the experiment of Figure \ref{fig:experiment1} and observe in Figure \ref{fig:behaviouratzero} that the distribution $\sigma$ of $(u_j^b)_k$, {with} some $\delta >0$, appears to satisfy $\sigma(x) \simeq x^{-1/2 + \delta}$ as $x \to 0$. By the previous argument, this implies that Assumption \ref{assum:a} should hold with $\nu = 2$. This estimate coincides with the findings of the experiment that was shown in Figure \ref{fig:experiment1}. 
		
Note that the assumption that the gradient updates have a density with a mild pole is quite mild. Consider for example, a random neural network with all weight matrices assumed Gaussian. If we replace all ReLUs by identities, then we would conclude that the bias updates were normally distributed. Since the normal distribution has a bounded density, the assumptions would be satisfied in this case. The application of the ReLU activation function will set some updates to $0$. This, however, is inconsequential for our analysis due to our convention $0^{\dagger} = 0$.

\section{Experimental study of the main theorem} \label{sec:numExpMain}
	
\begin{figure}[tb]
	\includegraphics[width  =0.45\textwidth]{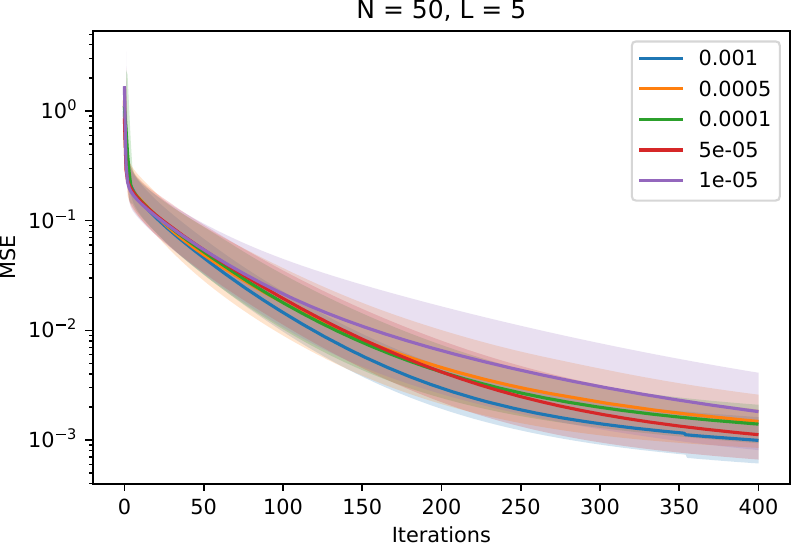} \includegraphics[width  =0.45\textwidth]{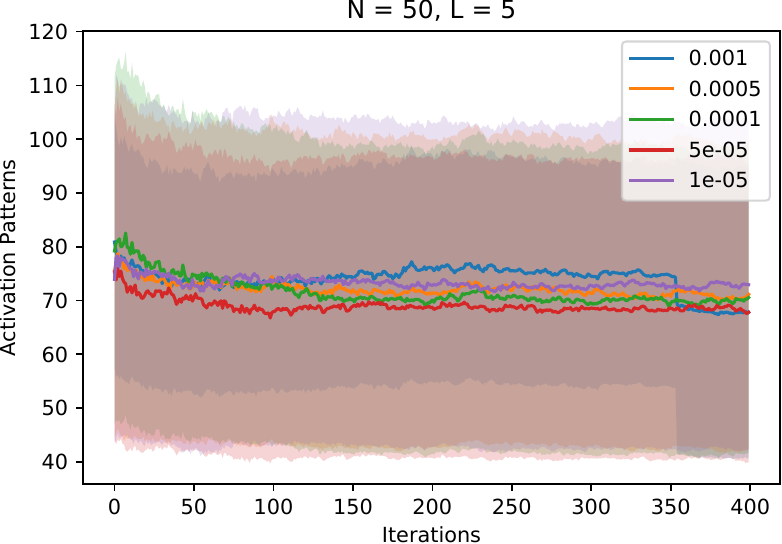}
	
	\includegraphics[width  =0.45\textwidth]{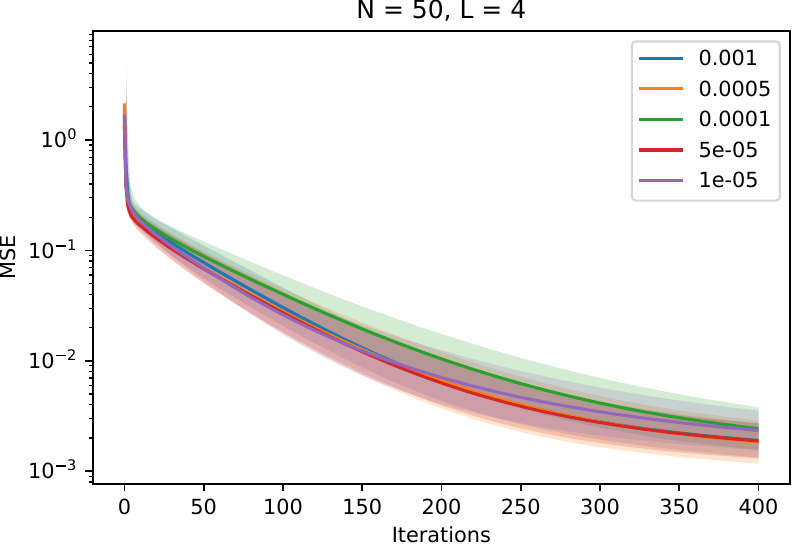} \includegraphics[width  =0.45\textwidth]{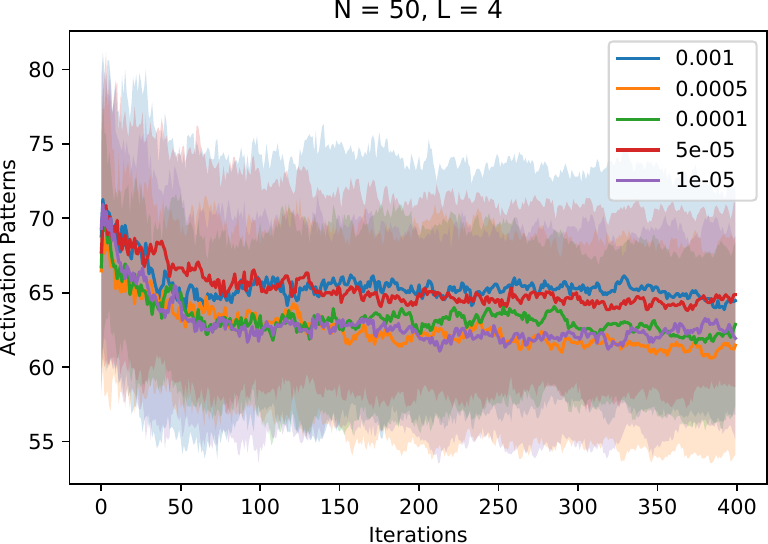}
\caption{Approximation of the function $x \mapsto 1-x^2/2$ on $[0,1]$ by neural networks with $50$ neurons per layer and $5$ layers (top row) and $50$ neurons per layer and $5$ layers (bottom row). Each matrix-vector multiplication during training is perturbed by relative noise of size $\{10^{-5}, 5\cdot 10^{-5}, 10^{-4}, 5\cdot 10^{-4}, 10^{-3} \}$. All results are averaged over fifteen runs, relative change of one standard deviation is shown as shaded area.} \label{fig:experiment1x2}
\end{figure}

\begin{figure}[tb]
\includegraphics[width  =0.45\textwidth]{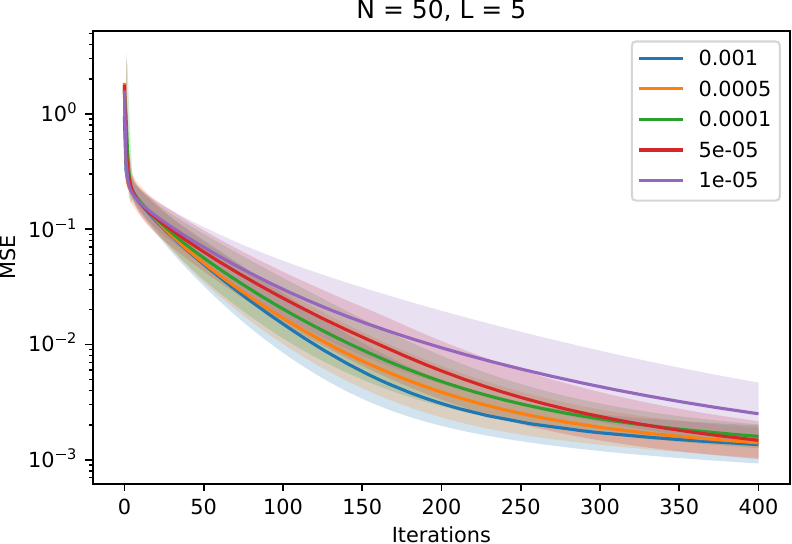} \includegraphics[width  =0.45\textwidth]{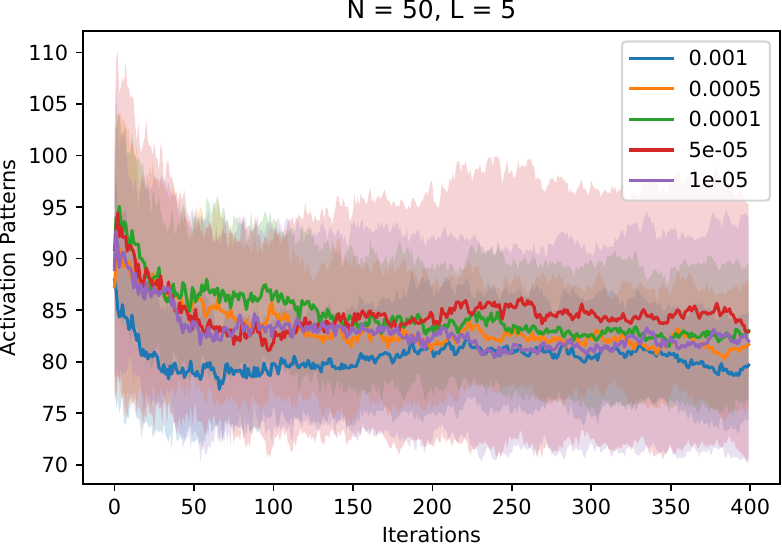}

\includegraphics[width  =0.45\textwidth]{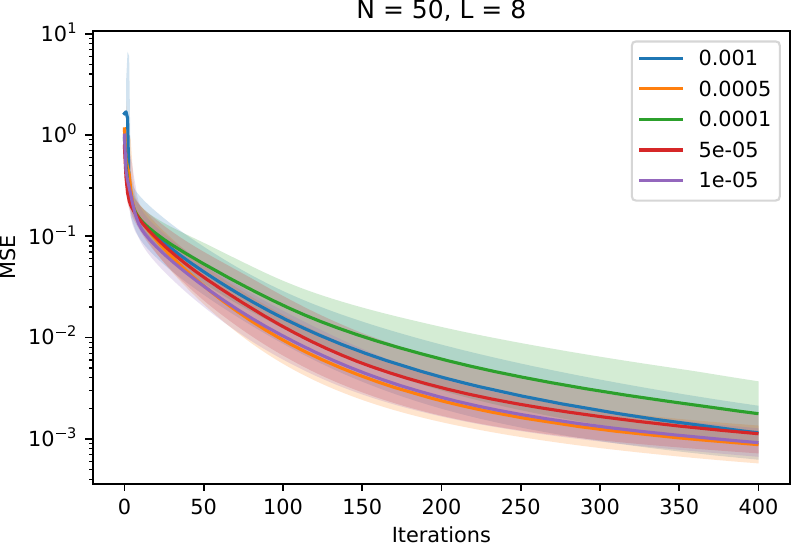} \includegraphics[width  =0.45\textwidth]{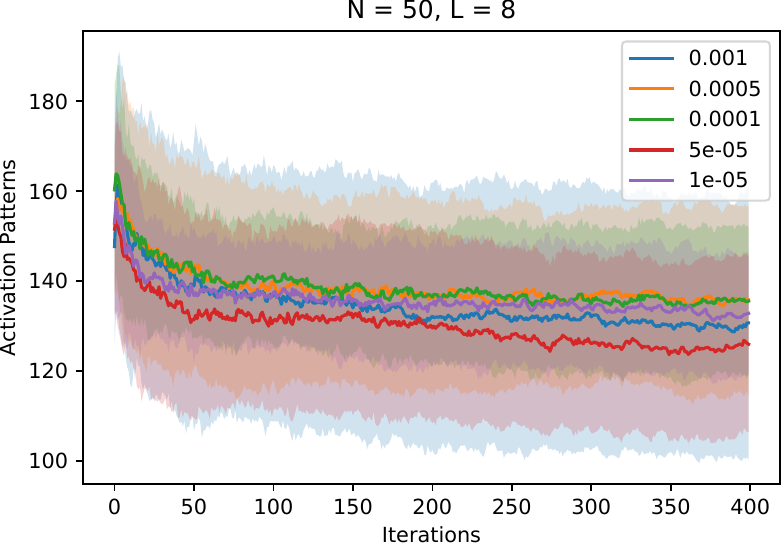}
\caption{Approximation of the cosine function on $[0,1]$ by neural networks with $50$ neurons per layer and $5$ layers (top row) and $50$ neurons per layer and $8$ layers (bottom row). Each matrix-vector multiplication during training is perturbed by relative noise of size $\{10^{-5}, 5\cdot 10^{-5}, 10^{-4}, 5\cdot 10^{-4}, 10^{-3} \}$. All results are averaged over fifteen runs, relative change of one standard deviation is shown as shaded area.} \label{fig:experiment1sin}
\end{figure}

\begin{figure}[tb]
	\includegraphics[width = 0.45\textwidth]{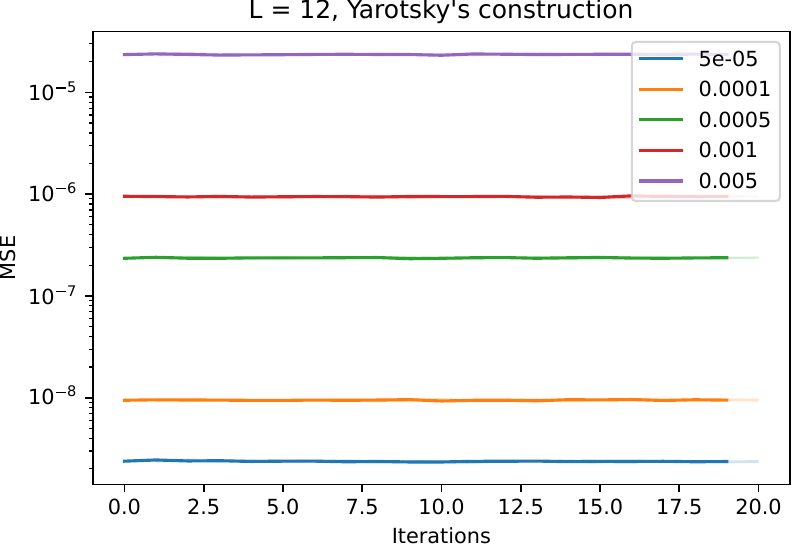} \includegraphics[width = 0.45\textwidth]{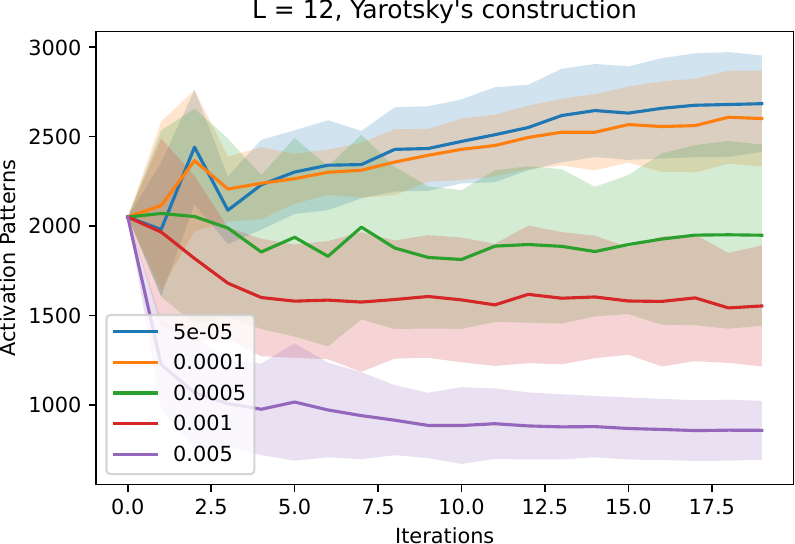}
	
	\includegraphics[width = 0.45\textwidth]{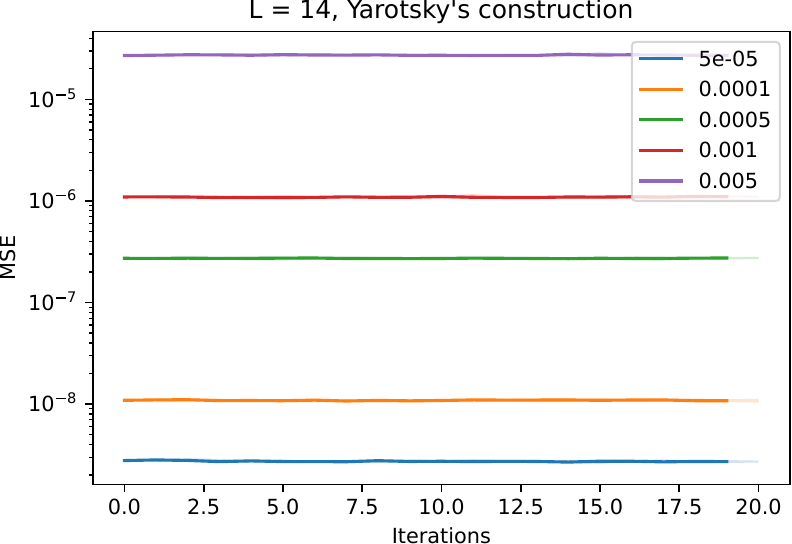} \includegraphics[width = 0.45\textwidth]{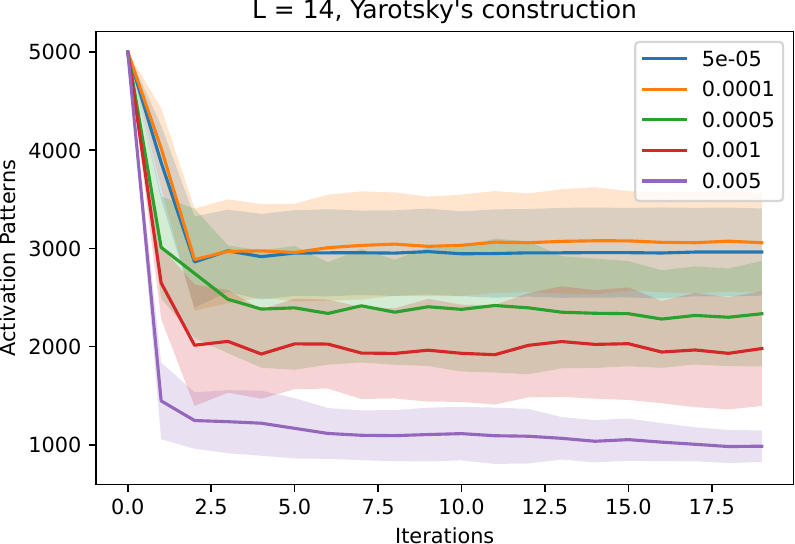}
	\caption{Training of neural networks initialised with many pieces according to \eqref{eq:yarostkyConstruction}. In the top row 12 layers are used, in the bottom row 14 layers are used. Each matrix-vector multiplication during training is perturbed by relative noise of size $\{5\cdot 10^{-5}, 10^{-4}, 5\cdot 10^{-4}, 10^{-3}, 5\cdot 10^{-3} \}$.
	All results are averaged over fifteen runs; the relative change of one standard deviation is shown as shaded area.} \label{fig:YarotskyNeuralNetwork}
\end{figure}

	In this section, we study the number of pieces of NN realisations in practical scenarios as well as the effect of the accuracy in floating-point computations. 
	Computing the number of pieces of a NN by identifying regions with gradients differing by a threshhold is extremely sensitive to the choice of the threshhold, and therefore does not yield reliable results. 	 
	Instead, we will count so-called activation regions, which 
	were considered as a proxy for affine pieces in \cite{Hanin2019}. 
	Indeed, the number of activation regions is an upper bound on the number of affine pieces.
	We refer to \cite[Definition 1]{Hanin2019} for the definition of activation regions and \cite[Lemma 3]{Hanin2019} for the connection between activation regions and affine pieces.

	In our first experiment, depicted in Figures \ref{fig:experiment1x2} and \ref{fig:experiment1sin}, we train NNs with various numbers of layers and neurons to approximate the functions $x \mapsto 1-x^2/2$ and $x\mapsto \cos(x)$, respectively. For that, the mean square loss on $500$ randomly chosen training points is minimised via gradient descent. As a step size, we use $\lambda_n = 0.02/(1+\sqrt{n}/8)$ at iteration $n$. This step size rule was empirically found to yield the fastest convergence. The NNs are randomly initialised according to the He initialisation \cite{he2015delving}. 
	In the computation of the gradient as well as when evaluating the NNs, we relatively perturb the output of each matrix-vector multiplication with relative noise the amplitude of which will be specified below.

	In Figures \ref{fig:experiment1x2} and \ref{fig:experiment1sin}, we observe that the number of activation regions \emph{always stays far below the upper bound of Theorem \ref{thm:numberofpiecesinfulliteration}}. However, there is no clear effect of the accuracy in the computations. 
	Note that, in the regime far away from the upper bound of Theorem \ref{thm:numberofpiecesinfulliteration}, it is conceivable that noise in the updates of the biases can even create new affine pieces, since, in this case, as seen in Remark \ref{rem:ProbEstimateOnNumberOfPieces}, the generation of affine pieces is not necessarily a low probability event. Hence, it should not be surprising that no clear effect of the level of numerical accuracy is visible in this regime. 
	Moreover, we expect that in addition to the accuracy of the computations, there are further reasons which prevent the creation of many pieces in practice. 
	To make the effect of the accuracy in floating-point arithmetic visible, we need to enter a regime, where the underlying NNs already have a lot of affine pieces and then observe the effect of the accuracy on the number of pieces. 
	We propose the following example:
	As famously shown in \cite[Proposition 2]{Yarotsky2017}, the function $x \mapsto x^2$ can be very efficiently approximated by ReLU NNs. The NN with $L$ layers and four neurons per layer given by
	\begin{align}\label{eq:yarostkyConstruction}
		b_\ell &\coloneqq 4^{1-\ell} \left(0, 0, -1, -2 \right)^T \text{ for } \ell = 1, \dots, L-1, \qquad b_L \coloneqq 0,\\
		A_1 &\coloneqq \left(1, 2, 2, 2\right)^T, \qquad  A_\ell \coloneqq \left(\begin{array}{cccc}
			1 & -1 & 2 & -1\\
			0 & -1/2 & 1 & -1/2\\
			0 & -1/2 & 1 & -1/2\\
			0 & -1/2 & 1 & -1/2
		\end{array}\right) \text{ for }\ell = 2, \dots, L-1, \\
	A_L &\coloneqq (1, -1 , 2, -1), \nonumber
	\end{align}
	satisfies $|\mathrm{R}(\Phi)(x) - x^2| \leq 4^{-L}$ for all $x \in [0,1]$ and has $2^{L-1}$ affine pieces.
	
	We initialise a NN in the way above, and train it with step size $\lambda_n = 0.1/(1+\sqrt{n})$, for $n = 1, \dots, 200$, on training data which is comprised of 5'000 uniformly randomly sampled elements of $[0,1]$ with labels described by the function $x \mapsto c x^2$, where $c = 0.99999$. Note that, simply changing $A_L$ to $\widetilde{A}_L = c \cdot (1, -1 , 2, -1)$ would already achieve zero training error. Hence, the target function is very close to the initialisation. 
	In this example, we \emph{clearly observe the effect of the numerical accuracy}, with high accuracy computations maintaining activation regions much better and in some cases even increasing the number of activation regions further. The results are depicted in Figure \ref{fig:YarotskyNeuralNetwork}.

	We observe in Figure \ref{fig:YarotskyNeuralNetwork} that if higher accuracy computations are used, then fewer activation regions are lost during the iteration compared to iterations with low accuracies.
	In all iterations, the mean squared error does \emph{not} significantly change during the iteration.

	\bibliographystyle{abbrv}
	\bibliography{references}
	\appendix
	
\section{Appendix}

\subsection{Example of unstable neural network}\label{sec:exUnstable}

Here, we introduce the notion of a \emph{finite precision realisation} of a given neural network $\Phi$ and give an example of how numerical instability can affect such NNs.
We assume the floating-point addition and multiplication to be defined according to~\eqref{eq:floating-point-ops},
where
the the minima are well defined since $\mathbb{M}$ has no accumulation points.
Further, we
consider the floating-point realisation of matrix-vector multiplication
that is given by
$$
(A \times_{\mathbb{M}} b)_j \coloneqq (A_{j,1} \times_{\mathbb{M}} b_1) +_{\mathbb{M}} (A_{j,2} \times_{\mathbb{M}} b_2) +_{\mathbb{M}} \dots +_{\mathbb{M}} (A_{j,n} \times_{\mathbb{M}} b_n),
$$
for any matrix $A \in \R^{m\times n}$ and a vector $b \in \R^n$ with $m,n\in\N$,
where  $j = 1, \dots m$ and the summation is performed from left to right.
Next, we define the finite precision realisation of a neural network.
\begin{definition}
	Let $\mathbb{M} \subset \R$ be a set without accumulation points. For a neural network $\Phi$ with input dimension $d\in \N$, $L \in \N$ layers, output dimension $N_L \in \N$ and an activation function $\varrho_{\mathbb{M}}\colon \mathbb{M} \to \mathbb{M}$, we define the \emph{finite precision realisation} by 
	\begin{align*}
		\mathrm{R}_{\mathbb{M}}(\Phi) \colon {\R^d} \to \mathbb{M}^{N_{L}}, x \mapsto x^{(L)},
	\end{align*}
	where $x^{(L)}$ is defined as 
	\begin{equation}
		\begin{split}
			x^{(0)} &\coloneqq x, \\
			x^{(j)} &\coloneqq \varrho_{\mathbb{M}}(A_{\ell} \times_{\mathbb{M}} x^{(j-1)} +_{\mathbb{M}} b_\ell) \quad \text{ for } j = 1, \dots, L-1,\\
			x^{(L)} &\coloneqq A_{L} \times_{\mathbb{M}} x^{(L-1)} +_{\mathbb{M}} b_{L}.
		\end{split}
		\label{eq:NetworkSchemeFinitePrecision}
	\end{equation}
\end{definition}
\begin{remark}
	Observe that in this section we only require the operations of the NNs to be elements of $\mathbb{M}$ but not the weights. This greatly simplifies the mathematical argument while still allowing NNs to exhibit instabilities.
\end{remark}

Having defined finite precision operations we continue by presenting a neural network $\Phi$ such that for every non-negative input $x$ it holds that
$$
|\Realization(\Phi)(x) - \Realization_{\mathbb{M}}(\Phi)(x)| = |\Realization(\Phi)(x)|= |x|.
$$
This implies that the relative error of passing to the finite precision network is 1. 


To illustrate numerical instability, we introduce the following neural network $\Phi$: Let $a \geq 1$, $N, L \in \N$ and
\begin{align*}
	A_1 = \left(\begin{array}{c}
		1\\
		a\\
		\vdots\\
		a
	\end{array}\right), \qquad
	A_{\ell} = \left(\begin{array}{cccc}
		1 & 0 & \dots & 0\\
		0 & a & \dots & a\\
		0 & \vdots & \ddots & \vdots\\
		0 & a & \dots & a
	\end{array}\right), \qquad  \text{ for } \ell =2, \dots L-3
\end{align*}
and
\begin{align*}
	A_{L-2} = \left(\begin{array}{cccc}
		1 & 0 & \dots & 0\\
		0 & a & \dots & a
	\end{array}\right), \qquad 
	A_{L-1} = \left(\begin{array}{cc}
		1 & a \\
		0 & a\\
	\end{array}\right), \qquad A_{L} = \left(\begin{array}{cc}
		1 & -1
	\end{array}\right).
\end{align*}
Furthermore, all biases satisfy $b_{\ell} = 0$ for $\ell =1, \dots, L$. We call the neural network defined above $\Phi_{a, N, L}$, an illustration is provided below in Figure \ref{fig:unstableNN}.
\begin{figure}[htb]
	\centering
	\includegraphics[width = 0.7\textwidth]{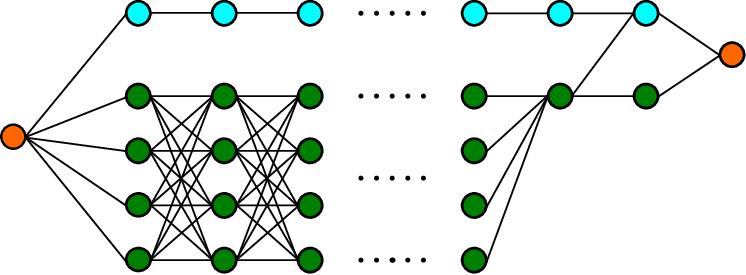}
	\normalsize
	\put(-285,32){$a$}
	\put(-285,46){$a$}
	\put(-285,58){$a$}
	\put(-285,70){$a$}
    \put(-285,95){$1$}
    \put(-245,115){$1$}
    \put(-209,115){$1$}
    \put(-102,115){$1$}
    \put(-66,115){$1$}
    \put(-25,103){$1$}
    \put(-28,77){-$1$}
    \put(-63,78){$a$}
    \put(-63,98){$a$}
    \put(-108,77){$a$}
    \put(-108,60){$a$}
    \put(-108,43){$a$}
    \put(-108,26){$a$}
	\caption{The unstable neural network $\Phi_{a, N, L}$. As defined earlier, all weights between green neurons equal $a$.}
	\label{fig:unstableNN}
\end{figure}
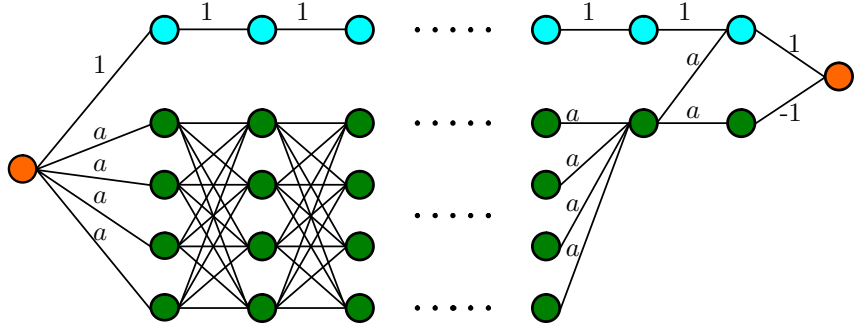

\begin{proposition}\label{prop:unstablenetComplete}
	Let $a = 2^r \geq  1$ for $r \in \Z$ and let $\mathbb{M}$ be as in \eqref{eq:defOfM}, with $e_{\min } \leq 0, e_{\max } > 53+r$, $p= 53$ and hence $\eps = 2^{-53}$. Furthermore, let $\varrho(x) = \max\{0,x\}$ and $\varrho_{\mathbb{M}}(z) \coloneqq \varrho(z)$ for all $z \in \mathbb{M}$. Additionally, let $L\in \N$ and $N \in \N$ be such that $(L-3)\log_{2}(N-1) + (L-1)\log_{2}(a-2\eps) \geq 54$, then the neural network $\Phi_{a, N,L}$ as defined above satisfies 
	$$
	|\Realization(\Phi_{a, N,L})(x) - \Realization_{\mathbb{M}}(\Phi_{a, N,L})(x)| = |\Realization(\Phi_{a, N,L})(x)|= |x|
	$$
	for all $x \in  \mathbb{M}$ such that $x \geq 0$.
\end{proposition}
\begin{proof}
	It follows from the definition of the NN $\Phi_{a, N,L}$ that $\Realization(\Phi)(x) = x$. Thus, to prove the theorem we have to show that
	\begin{align}\label{eq:ThisIsToShow}
		\Realization_{\mathbb{M}}(\Phi_{a, N,L})(x) = 0
	\end{align}
	for all nonnegative $x$. Let, for $x\geq 0$, $x^{(0)}, \dots, x^{(L)}$ be as in \eqref{eq:NetworkSchemeFinitePrecision}. The construction of $\Phi_{a, N,L}$ implies that $x^{(L)} = 0$ if $(x^{(L-1)})_1 = (x^{(L-1)})_2$. Additionally, $(x^{(L-1)})_1 = (x^{(L-1)})_2$ follows if $(x^{(L-2)})_1 < a (x^{(L-2)})_2 \eps/2$.
	
	As a consequence, we have to find a lower bound of $(x^{(L-2)})_2$. To begin with, we observe that for $x \in \mathbb{M}$ the multiplication with $a$ is exact and thus $(x^{(2)})_k = a x$ for all $k = 2,\dots, N$. Also, we observe that for $j = 2, \dots, L-3$ it holds for $k\in \{2, \dots, n\}$ that
	\begin{align*}
		&(N-1) \, a \max_{i \in \{2, \ldots, N-1\}}(x^{(j-1)})_i + \eps (N-1) \, a \max_{i \in \{2, \ldots, N-1\}}(x^{(j-1)})_i\\
		&\qquad  \geq  (A_j \, x^{(j-1)})_k \\
		&\qquad \geq (N-1) \, a \min_{i \in \{2, \ldots, N-1\}}(x^{(j-1)})_i -  \eps (N-1) \, a \max_{i \in \{2, \ldots, N-1\}}(x^{(j-1)})_i.
	\end{align*}
	This inequality results from the assumption that in the summation of each row of $A_j$ we create $(N-2)$ times an error at most of size $\eps \cdot \max_{i \in \{2, \ldots, N-1\}}(x^{(j-1)})_i$.
	This leads to 
	$$
		(x^{(L-3)})_k \geq a \, (N-1)^{L-4} (a - 2\eps)^{L-4} x
	$$
	and
	$$
		(x^{(L-2)})_2 \geq a \, (N-1)^{L-3} (a - 2\eps)^{L-3} x
	$$
	for all $k \in \{2,\dots, N\}$. Due to the construction of our NN, we also have that $(x^{(L-2)})_1 = x$. 
	
	Thus, the inequality $(x^{(L-2)})_1 < a \, (x^{(L-2)})_2  \, \frac{\eps}{2}$ holds if 
	$$
		\frac{1}{2}a^2 (N-1)^{L-3} (a - 2\eps)^{L-3} > \eps^{-1}. 
	$$
	Note that
\begin{align}
	\label{eq:nonimportentinequality}
	\frac{1}{2}a^2 (N-1)^{L-3} (a - 2\eps)^{L-3} > \frac{1}{2} (N-1)^{L-3} (a - 2\eps)^{L-1}
\end{align}
	and hence, it suffices to show that the right-hand side of \eqref{eq:nonimportentinequality} is bounded below by $\epsilon^{-1} = 2^{53}$. 
	Taking logarithms, implies that this inequality is implied by 
	$$
		(L-3)\log_{2}(N-1) + (L-1)\log_{2}(a-2\eps) - \log_{2}(2)> 53,
	$$
	which completes the proof.
\end{proof}

\subsection{Auxiliary results}

\subsubsection{Proof of Lemma \ref{lem:1}}\label{sec:auxResults}

\begin{proof}
	We denote the affine pieces of $h$ on $[0,c]$ by $(I_{p})_{p = 1}^P$. Here we choose the pieces as semi-open intervals $I_{p} = (r_p, s_p]$ so that they are disjoint. Moreover, we choose $x_p \in I_p$ for $p = 1, \dots, P$.	We claim that 
	\begin{align}\label{eq:theClaim}
		\sum_{p= 1}^P \lambda(h(I_p)) \geq t \lambda(A).
	\end{align}
	To prove this, we consider the auxiliary function 
	\begin{align*}
		\tilde{h} \colon [0,1] &\to A \times \left(\{1, \dots, P\} \cup \{\infty\}\right),\\
		\tilde{h}(x) &\coloneqq \left( h(x),  \#\{z \in [0, x] \colon h(z) = h(x) \} \right).
	\end{align*}
	Note that for $x \in [0,1]$ either $\#\{z \in [0, x] \colon h(z) = h(x) \} \leq P$ or $\#\{z \in [0, x] \colon h(z) = h(x) \} = \infty$, and hence $\tilde{h}$ is well defined. 
	We have by assumption that 
	$$
		\tilde{h}([0,1]) \supset A \times \{1, \dots, t\}.
	$$
	If we equip $A \times (\{1, \dots, P\} \cup \{\infty\})$ with the measure $\lambda \times \nu$, where $\lambda$ is the Lebesgue and $\nu$ is the counting measure, then we conclude by the monotonicity of measures that 
	$$
		\left(\lambda \times \nu\right)\left(\tilde{h}([0,1])\right) \geq \lambda(A) t.
	$$
	It is easy to see that 
	\begin{align}\label{eq:crucialEstimate}
		\lambda(h(I_p)) \leq (\lambda \times \nu)\left(\tilde{h}(I_p)\right) \leq \sup_{\substack{A_1, \dots, A_{P+1}\subset I_p,\\ \text{ measurable}}} \sum_{q= 1}^{P+1} \lambda(h(A_q)),
	\end{align}
	If $h'(x_p) \neq 0$ on $I_p$, then the right-hand side of \eqref{eq:crucialEstimate} equals $\lambda(h(I_p))$. If $h'(x_p) = 0$, then the right-hand side of \eqref{eq:crucialEstimate} is $0$ and by the non-negativity of the Lebesgue measure the left hand-side is $0$ as well. We conclude that 
	\begin{align}
		\label{eq:EqualityOnPieces}
		\lambda(h(I_p)) = (\lambda \times \nu)(\tilde{h}(I_p)).
	\end{align}
	By the subadditivity of measures and \eqref{eq:EqualityOnPieces}, we have that 
	$$
		\lambda(A) t \leq \sum_{p=1}^P (\lambda \times \nu)\left( \tilde{h}(I_p)\right) = \sum_{p=1}^P \lambda (h(I_p)), 
	$$
	which yields \eqref{eq:theClaim}.
	Finally, by standard properties of the Lebesgue measure under Lipschitz maps, we have that 
	$$
	\sum_{p= 1}^P \lambda(h(I_p)) \leq \sum_{p= 1}^P h'(x_p) \lambda(I_p) = \|h'\|_1 \leq c \|h'\|_\infty, 
	$$
	where we used that $I_p$ are disjoint. 
\end{proof}

\subsubsection{Proof of Proposition \ref{prop:gd}}\label{sec:proofprop3}

\begin{proof}

	\textbf{Part 1: (Proof of \eqref{eq:probabilityEstimate})}
	Let $\kappa \subset [0,1]^d$ be a straight line. Recalling equation~\eqref{eq:pertbiasweight}, we observe that
	\begin{align*}
		\widehat{\bold{A}}_j\mathbbm{1}_{\hat{\eta}_{j-1}\geq0}
		\,
		\hat{\eta}_{j-1}(x)
		+
		\widehat{\bold{b}}_j&=\widehat{\bold{A}}_j\mathbbm{1}_{\hat{\eta}_{j-1}\geq 0}
		\,
		\hat{\eta}_{j-1}(x)+(b_j-\lambda \hat{u}_j^b)
		\, .
	\end{align*}
Inserting $\hat{u}_j^b=(I+\varepsilon_j\Theta_j^b)u_j^b$ from equation~\eqref{eq:update} leads to
\begin{align*}
		\widehat{\bold{A}}_j\mathbbm{1}_{\hat{\eta}_{j-1}\geq0}
		\,
		\hat{\eta}_{j-1}(x)+\widehat{\bold{b}}_j&
		=
		\widehat{\bold{A}}_j\mathbbm{1}_{\hat{\eta}_{j-1}\geq0}
		\,
		\hat{\eta}_{j-1}(x)+(\bold{b}_j- \lambda \varepsilon_j\Theta_j u_j^b)\\
		&\eqqcolon\widehat{\bold{A}}_j\mathbbm{1}_{\hat{\eta}_{j-1}\geq0}
		\,
		\hat{\eta}_{j-1}(x)+\bold{b}_j+\theta_{j}.
\end{align*}
	Let us introduce for $k \in \{1, \ldots, N_j\}$ the random variable
	$$
		Q_{\widehat{\bold{A}}_{j}, \hat{\eta}_{j-1}, \theta_{j},k}
		:=
		\#\{
			x \in \kappa \colon
			\big(
				\widehat{\bold{A}}_{j}\mathbbm{1}_{\hat{\eta}_{j-1}\geq 0}
				\,
				\hat{\eta}_{j-1}(x)+\bold{b}_j
				+
				\theta_{j}
			\big)_k=0
		\}
	$$
	and observe that
	$$
		\# \left(\hat{\omega}_{j,k} \setminus \omega_{j,k} \right) \leq  	Q_{\widehat{\bold{A}}_{j}, \hat{\eta}_{j-1}, \theta_{j},k}
	$$
	since the number of breakpoints added after applying the activation function is less or equal than the number of $x \in \kappa$ satisfying $(\widehat{\bold{A}}_{j}\mathbbm{1}_{\hat{\eta}_{j-1}\geq0}\hat{\eta}_{j-1}(x)+\bold{b}_j)_k+(\theta_{j})_k= 0$.
	Per Definition \ref{def:GD}, we have that $\widehat{\bold{A}}_{j}$, $\hat{\eta}_{j-1}$, and $\theta_{j}$ are independent random variables which allows the following computation for $q \in \N$
	\begin{align}
		\mathbb{P}(\# \left(\hat{\omega}_{j,k} \setminus \omega_{j,k} \right)\geq q) &\leq
		\mathbb{P}_{\widehat{\bold{A}}_{j},\theta_{j}, \hat{\eta}_{j-1}}(Q_{\widehat{\bold{A}}_{j},\theta_{j},k}\geq q)\nonumber\\
		&=\mathbb{E}_{\widehat{\bold{A}}_{j}, \hat{\eta}_{j-1}}\mathbb{E}_{\theta_{j}}\mathbbm{1}_{Q_{\widehat{\bold{A}}_{j}, \hat{\eta}_{j-1}, \theta_{j},k} \geq q}\nonumber\\
		&=\mathbb{E}_{\widehat{\bold{A}}_{j}, \hat{\eta}_{j-1}} \mathbb{P}\left(Q_{\widehat{\bold{A}}_{j}, \hat{\eta}_{j-1}, \theta_{j},k} \geq q \right). \label{eq:weWantToApplySomethingToThisToFinishTheProof}
	\end{align}		
	By Assumption \ref{assum:b} we have that for all $x \in \Omega$
	\begin{align}
		1 &\geq |\nabla_x (\eta_j)_k| \nonumber\\
		&= \left|\nabla_x \left(\widehat{\bold{A}}_{j}\mathbbm{1}_{\hat{\eta}_{j-1}\geq 0}\hat{\eta}_{j-1}(x)+\bold{b}_j\right)_k\right|\nonumber\\
		&= \left|\nabla_x \left(\widehat{\bold{A}}_{j}\mathbbm{1}_{\hat{\eta}_{j-1}\geq 0}\hat{\eta}_{j-1}(x)\right)_k\right|. \label{eq:consOfAssumB}
	\end{align}

	We pick $v,  w \in \R^d$ with $\|v\|= 1$ such that 
	$\kappa = \{tv + w \colon  t \in [0, \mathcal{L}(\kappa)] \}$ where $\mathcal{L}(\kappa)$ is the length of $\kappa$. Next, we define 
	\begin{align*}
		h\colon [0, \mathcal{L}(\kappa)] &\to [0,1]^d\\
		t &\mapsto  \left(\widehat{\bold{A}}_{j}\mathbbm{1}_{\hat{\eta}_{j-1}\geq 0}
		\,
		\hat{\eta}_{j-1}(tv + w )\right)_k
		\, .
	\end{align*}
	Then, by the chain rule and \eqref{eq:consOfAssumB}, we have that $\|h'\|_\infty \leq 1$.
	Moreover, 
	\begin{align*}
		Q_{\widehat{\bold{A}}_{j}, \hat{\eta}_{j-1}, \theta_{j},k} = \#\{x \in [0, \mathcal{L}(\kappa)]: h(x) +(\theta_{j})_k=0\}.
	\end{align*}
	Since by construction $(\theta_{j})_k$ is uniformly distributed on the interval $ [-\lambda \varepsilon_j  |(u_j^b)_k|/2, \lambda \varepsilon_j  |(u_j^b)_k|/2]$, we conclude
	from
	Remark \ref{rem:ProbEstimateOnNumberOfPieces} that 
	\begin{align*}
		\mathbb{P}\left(Q_{\widehat{\bold{A}}_{j}, \hat{\eta}_{j-1}, \theta_{j},k} \geq q \right) \leq \frac{2 \mathcal{L}(\kappa)}{\lambda } \cdot \left(\varepsilon_j \, q \, |(u_j^b)_k|\right)^{-1}
	\end{align*}
	applied
	to \eqref{eq:weWantToApplySomethingToThisToFinishTheProof} yields \eqref{eq:probabilityEstimate}.
	
	\textbf{Part 2: (Proof of \eqref{eq:expectationestimate})}
	We recall the following upper bound on the total number of pieces of the realisation of an arbitrary NN:

	\begin{theorem}[\cite{Telgarsky2015RepresentationBO}]\label{thm:1}
		Let $L \in \N$ and $\sigma$ be piecewise affine with $p$ pieces. Then, for every NN $\Phi$ with $d=1, N_L=1$ and $N_1, \ldots, N_{L-1} \leq N$, we have that $\mathrm{R}(\Phi)$ has at most $(pN)^{(L-1)}$ affine pieces.
	\end{theorem}

	Invoking Theorem \ref{thm:1} results in $\max_{k}\{\# \left(\hat{\omega}_{j,k} \setminus \omega_{j,k} \right)\} \leq N^{j-1}$ for $j = 1, \dots, L$. We observe that
		\begin{align*}
			\mathbb{E}(\# (\hat{\omega}_{j,k} \setminus \omega_{j,k} ))&=\sum_{q=1}^{N^{j}}\bigg[\mathbb{P}(\# \left(\hat{\omega}_{j,k} \setminus \omega_{j,k} \right) \geq q)-\mathbb{P}(\# \left(\hat{\omega}_{j,k} \setminus \omega_{j,k} \right) \geq q+1)\bigg]q\\
			&=\sum_{q=1}^{N^{j}}\mathbb{P}(\# \left(\hat{\omega}_{j,k} \setminus \omega_{j,k} \right) \geq q)q -\sum_{q=2}^{N^{j}+1}\mathbb{P}(\# \left(\hat{\omega}_{j,k} \setminus \omega_{j,k} \right) \geq q)(q-1)\\
			&=\sum_{q=1}^{N^{j}}\mathbb{P}(\# \left(\hat{\omega}_{j,k} \setminus \omega_{j,k} \right) \geq q)(q-(q-1))
			\, .
	\end{align*}
	Next, we invoke \eqref{eq:probabilityEstimate} and obtain that
	\begin{align*}
		\mathbb{E}(\# (\hat{\omega}_{j,k} \setminus \omega_{j,k} ))&\leq\sum_{q=1}^{N^{j}}\frac{2\mathcal{L}(\kappa)}{\lambda \varepsilon_j |(u_j^b)_k| q}(q-(q-1))\\
			&=\frac{2 \mathcal{L}(\kappa)}{\lambda \varepsilon_j |(u_j^b)_k|} \sum_{q=1}^{N^{j}}\frac{1}{q}\\
			&\leq \frac{2\mathcal{L}(\kappa)}{\lambda \varepsilon_j |(u_j^b)_k|} \left(1+\int_{1}^{N^{j}} \frac{1}{x} dx\right)\\\
			&\leq \frac{2\mathcal{L}(\kappa)}{\lambda \varepsilon_j |(u_j^b)_k|} (1+\ln(N^{j})) \leq \frac{\mathcal{L}(\kappa)}{2 \lambda \varepsilon_j |(u_j^b)_k|} (\ln(e N^{j})) \leq \frac{2\mathcal{L}(\kappa)}{\lambda  \varepsilon_j |(u_j^b)_k|} (\ln(N^{j+1}))
			\, ,
		\end{align*}
		where we used that $e \leq 3 \leq N$. This completes the proof of \eqref{eq:expectationestimate}.
\end{proof}

\subsection{Proof of Theorem \ref{thm:upperBoundOnPieces}}\label{sec:proofOfupperBoundOnPieces}

\begin{proof}

\textbf{Part 1:} We start by proving the result for $j' = L$. Using the notation of Proposition \ref{prop:gd}, we start by assuming that
\begin{align}\label{eq:weWillProveThisAtTheEnd}
	\mathrm{pieces}(\mathrm{R}(\widehat{\Phi}^\varepsilon), \kappa)) \leq 1 + \sum_{j \in [L-1]} \sum_{k \in N_j}  \# \left(\hat{\omega}_{j,k} \setminus \omega_{j,k}\right),
\end{align}
if $\omega_{j,k}$ is chosen as $\omega_{j,k} = \bigcup_{k \in N_{j-1}} \hat{\omega}_{j-1,k}$ for $j > 1$ and $\omega_{1,k} = \emptyset$.

We will prove \eqref{eq:weWillProveThisAtTheEnd} at the end this step of this proof to not distract from the main argument.
By \eqref{eq:weWillProveThisAtTheEnd} and the linearity of the expected value, we obtain, 
$$
	\mathbb{E}(\mathrm{pieces}(\mathrm{R}(\widehat{\Phi}^\varepsilon), \kappa))) \leq 1+ \sum_{j \in [L-1]} \sum_{k \in N_j} \mathbb{E}(\# \left(\hat{\omega}_{j,k} \setminus \omega_{j,k}\right))
	\, .
$$
By Assumption \ref{assum:a}, it holds that if $(u_j^b)_k = 0$, then $\mathrm{R}((\widehat{\Phi}^\varepsilon)_j)_k = 0$ and therefore $\# \left(\hat{\omega}_{j,k} \setminus \omega_{j,k}\right) = 0$. Hence, invoking Proposition \ref{prop:gd}, we conclude, using Assumption \ref{assum:a}, that \begin{align*}
	\mathbb{E}(\mathrm{pieces}(\mathrm{R}(\widehat{\Phi}^\varepsilon), \kappa))) &\leq 1+ 2\mathcal{L}(\kappa) \lambda^{-1}\sum_{j =1 }^{L-1} \sum_{k \in N_j} \frac{1}{\varepsilon_j} |(u_j^b)_k|^{\dagger} \log(N^{j+1})\\
	& \leq 1+  2\mathcal{L}(\kappa) \lambda^{-1}\sum_{j =1 }^{L-1} \sum_{k \in N_j} \frac{j+1}{\varepsilon_j} |(u_j^b)_k|^{\dagger} \log(N)\\
	& \leq 1+  2\mathcal{L}(\kappa) \lambda^{-1} \frac{L}{\hat{\epsilon}_{L}} \sum_{j =1 }^{L-1} \sum_{k \in N_j} |(u_j^b)_k|^{\dagger} \log(N)\\
	& \leq 1+ 2\mathcal{L}(\kappa) \lambda^{-1}  c_0 \frac{L}{\hat{\epsilon}_{L}} N^\nu \log(N)
	\, .
\end{align*}
We complete the proof of the first part by showing \eqref{eq:weWillProveThisAtTheEnd}.
Note that since
\begin{align}\label{eq:douaapsk}
	\mathrm{R}(\widehat{\Phi}^\varepsilon)_{k'}
	=
	\sum_{k \in N_{L-1}} (\widehat{\bold{A}}_{L})_{k'k}
	\,
	\varrho(\mathrm{R}((\widehat{\Phi}^\varepsilon)_{L-1})_k) + (\widehat{\bold{b}}_L)_{k'}
	\, ,
\end{align}
we have that 
\begin{align*}
	\mathrm{pieces}(\mathrm{R}(\widehat{\Phi}^\varepsilon), \kappa)) \leq 1 + \#\bigcup_{k \in N_{L-1}} \hat{\omega}_{L-1,k}
	\, ,
\end{align*}
where we use that a piecewise affine function on a line has exactly one more affine piece than the smallest number of distinct points where the function is not affine. 

We will show by induction over $L$ that 
\begin{align}\label{eq:inductionHypothesis}
	\#\bigcup_{k \in N_{L-1}}\hat{\omega}_{L-1,k} \leq \sum_{j \in [L-1]} \sum_{k \in N_j}  \# \left(\hat{\omega}_{j,k} \setminus \omega_{j,k}\right)
\end{align}
which then yields the claim. For $L = 2$, the claim follows directly by the choice $\omega_{1,k} = \emptyset$.

Now let $L > 2$, then 
\begin{align}
	\bigcup_{k \in N_{L-1}}\hat{\omega}_{L-1,k} &=  \bigcup_{k \in N_{L-1}}\left(\hat{\omega}_{L-1,k} \setminus \omega_{L-1,k}\right) \cup \omega_{L-1,k}\nonumber\\
	& = \omega_{L-1,1} \cup \bigcup_{k \in N_{L-1}}\left(\hat{\omega}_{L-1,k} \setminus \omega_{L-1,k}\right),\label{eq:set-estimate}
\end{align}
since $\omega_{L-1,1} = \omega_{L-1,k}$ for all $k \in N_{L-1}$ by definition. 
We have that 
\begin{align*}
	\mathrm{R}((\widehat{\Phi}^\varepsilon)_{L-1})_{k'}
	=
	\sum_{k \in N_{L-2}} (\widehat{\bold{A}}_{L-1})_{k'k} \, \varrho(\mathrm{R}(\Phi_{L-2})_k) + (\widehat{\bold{b}}_{L-1})_{k'}
	\, ,
\end{align*}
and hence 
\begin{align}\label{eq:inclusionPropertyPfSingularities}
	\omega_{L-1,1} \subset \bigcup_{k \in N_{L-2}} \hat{\omega}_{L-2, k}
	\, .
\end{align}

Using \eqref{eq:inclusionPropertyPfSingularities} and applying the induction hypothesis \eqref{eq:inductionHypothesis} to \eqref{eq:set-estimate}, we obtain that
\begin{align*}
	\#\bigcup_{k \in N_{L-1}}\hat{\omega}_{L-1,k} &\leq \#\omega_{L-1,1}  + \#\bigcup_{k \in N_{L-1}}\left(\hat{\omega}_{L-1,k} \setminus \omega_{L-1,k}\right) \\
	&\leq \#\bigcup_{k \in N_{L-2}} \hat{\omega}_{L-2, k}  + \#\bigcup_{k \in N_{L-1}}\left(\hat{\omega}_{L-1,k} \setminus \omega_{L-1,k}\right) \\
	& \leq\sum_{j \in [L-2]} \sum_{k \in N_j}  \# \left(\hat{\omega}_{j,k} \setminus \omega_{j,k}\right) + \sum_{k \in N_{L-1}} \# \left(\hat{\omega}_{L-1,k} \setminus \omega_{L-1,k}\right)
	\, ,
\end{align*}
which yields \eqref{eq:inductionHypothesis}.

\textbf{Part 2:} Let $j'>0$. We have that 
\begin{align*}
	\mathrm{R}(\widehat{\Phi}^\varepsilon) =  \mathrm{R}(\widehat{\Phi}^{\varepsilon, c}_{j'}) \circ \varrho \circ \mathrm{R}(\widehat{\Phi}^\varepsilon_{j'}),
\end{align*}
where 
\begin{align}
	\widehat{\Phi}^{\varepsilon, c}_{j'} = \big((\widehat{\bold{A}}_{j'+1}, \widehat{\bold{b}}_{j'+1}), \dots, (\widehat{\bold{A}}_L, \widehat{\bold{b}}_L)\big).
\end{align}
For $q, p \in \N$ and two functions $f: \kappa \to \R^q$ and $g :\R^q \to \R^p$ such that $f$ has $r \in \N$ affine pieces on $\kappa$, and $g$ has at most $s\in \N$ many pieces along all possible lines it holds that $g \circ f \colon \kappa \to \R^q$ has at most $rs$ affine pieces, since on the image of each affine piece of $f$ no more than $s$ affine pieces can be generated by $g$. Note that $\varrho \colon \R^{N_{j'}} \to \R^{N_{j'}}$ has at most $2N$, pieces along each line, since every coordinate of a line can change sign at most once. Invoking again Proposition \ref{thm:1} as well as Part 1 of this proof and the monotonicity of the expected value therefore yields that 
\begin{align*}
 \mathbb{E}(\mathrm{pieces}(\mathrm{R}(\widehat{\Phi}^\varepsilon), \kappa)) &\leq  (2N)^{L-j'-1} \cdot 2 N_{j'} \cdot \left( 1+ 2\mathcal{L}(\kappa) \lambda^{-1}  c_0 \, \frac{j'}{\hat{\epsilon}_{j'}} N^\nu \log(N)\right)\\
 	& \leq (2N)^{L-j'} \cdot \left( 1+ 2\mathcal{L}(\kappa) \lambda^{-1} c_0 \, \frac{j'}{\hat{\epsilon}_{j'}} N^\nu \log(N)\right), 
\end{align*}
which completes the proof.
\end{proof}

\end{document}